\documentclass[11pt]{article}
\pdfoutput=1
\usepackage{enumerate}
\usepackage{pdfsync}
\usepackage[OT1]{fontenc}

\usepackage[usenames]{color}
\usepackage{smile}
\usepackage[colorlinks,
            linkcolor=red,
            anchorcolor=blue,
            citecolor=blue
            ]{hyperref}
\usepackage{fullpage}
\usepackage{hyperref}
\usepackage[protrusion=true,expansion=true]{microtype}
\usepackage{pbox}
\usepackage{setspace}
\usepackage{tabularx}
\usepackage{float}
\usepackage{wrapfig,lipsum}
\usepackage{enumitem}
\usepackage{microtype}
\usepackage{graphicx}
\usepackage{subfigure}
\usepackage{enumerate}
\usepackage{enumitem}
\usepackage{pgfplots}
\usepackage{mathtools}
\usetikzlibrary{arrows,shapes,snakes,automata,backgrounds,petri}
\usepackage[nice]{nicefrac}

\allowdisplaybreaks

\usepackage{enumitem}

\usepackage{xspace}
\usepackage{setspace}
\usepackage{latexsym,mathtools,multicol,calligra}
\usepackage{pstricks,listings,stackengine}

\definecolor{deepblue}{rgb}{0,0,0.5}
\definecolor{deepred}{rgb}{0.6,0,0}
\definecolor{deepgreen}{rgb}{0,0.5,0}
\definecolor{halfgray}{gray}{0.55}

\DeclarePairedDelimiterX\condtvInner[3]{(}{)}%
{#1, #2\delimsize| #3}

\DeclarePairedDelimiterX\condklInner[3]{(}{)}%
{#1\delimsize\| #2\delimsize| #3}

\newcommand{\dist}[3]{\ensuremath{{\mathsf{dist}}_{#1}\left(#2, #3\right)}}

\newcommand{\kl}[2]{\ensuremath{{\mathsf{KL}}\left(#1\|#2\right)}}
\newcommand{\condkl}[3]{\ensuremath{{\mathsf{KL}}\condklInner{#1}{#2}{#3}}}
\newcommand{\tv}[2]{\ensuremath{{\mathsf{TV}}\left(#1, #2\right)}}
\newcommand{\condtv}[3]{\ensuremath{{\mathsf{TV}}\condtvInner{#1}{#2}{#3}}}

\newcommand{\nm}[2]{\ensuremath{{\mathfrak{n}_{#1}\left(#2\right)}}}
\newcommand{\mm}[2]{\ensuremath{{\mathfrak{m}_0}\left(#1, #2\right)}}

\newcommand{\lossCEsingle}{\mathsf{CE_{\hyperref[eq:def:lossCEsingle_D]{sgl}}}}
\newcommand{\lossEMPCE}{{\mathsf{CE_{\hyperref[eq:def:lossEMPCE_D+R]{ptl}}}}}
\newcommand{\lossEMPSEL}{{\mathsf{SEL_{\hyperref[eq:def:lossEMPSEL]{ptl}}}}}
\newcommand{\lossCEfull}{\mathsf{CE_{\hyperref[eq:CE_full_def]{ful}}}}

\newcommand{\nll}{%
\ensuremath{{\hat\pi_{\mathsf{CE,{\hyperref[eq:def:lossCEsingle_D]{sgl}}}}}}%
}
\newcommand{\empce}{%
\ensuremath{{\hat{\pi}_\mathsf{CE,{\hyperref[eq:def:lossEMPCE_D+R]{ptl}}}}}%
}
\newcommand{\empsel}{%
\ensuremath{{\hat{\pi}_\mathsf{SEL,{\hyperref[eq:def:lossEMPSEL]{ptl}}}}}%
}
\newcommand{\fullkl}{%
\ensuremath{{\hat\pi_\mathsf{CE_{\hyperref[eq:CE_full_def]{ful}}}}}%
}

\newcommand{\dataone}{\textsf{H}{ard} \textsf{L}{abels}}
\newcommand{\datatwo}{\textsf{P}{artial} \textsf{SL}{s}}
\newcommand{\datathree}{\textsf{S}{oft} \textsf{L}{abels}}

\makeatletter
\newcommand*{\rom}[1]{\expandafter\@slowromancap\romannumeral #1@}
\makeatother
\title{\huge Towards the Fundamental Limits of Knowledge Transfer over Finite Domains}

\author
{
	Qingyue Zhao\thanks{Department of Computer Science and Technology, Tsinghua University; e-mail: {\tt zhaoqy19@mails.tsinghua.edu.cn}} 
	~~~and~~~
	Banghua Zhu\thanks{Department of Electrical Engineering and Computer Sciences, University of California, Berkeley; e-mail: {\tt banghua@berkeley.edu}}
}

\ifdefined\final
\usepackage[disable]{todonotes}
\else
\usepackage[textsize=tiny]{todonotes}
\fi

\begin{document}
\date{}
\maketitle

\begin{abstract}
    We characterize the statistical efficiency of knowledge transfer through $n$ samples from a teacher to a probabilistic student classifier with input space $\cS$ over labels $\cA$. We show that privileged information at three progressive levels accelerates the transfer. At the first level, only samples with hard labels are known, via which the maximum likelihood estimator attains the minimax rate $\sqrt{\nicefrac{|{\cS}||{\cA}|}{n}}$. The second level has the teacher probabilities of sampled labels available in addition, which turns out to boost the convergence rate lower bound to ${\nicefrac{|{\cS}||{\cA}|}{n}}$. However, under this second data acquisition protocol, minimizing a naive adaptation of the cross-entropy loss results in an asymptotically biased student. We overcome this limitation and achieve the fundamental limit by using a novel empirical variant of the squared error logit loss. The third level further equips the student with the soft labels (complete logits) on ${\cA}$ given every sampled input, thereby provably enables the student to enjoy a rate $\nicefrac{|{\cS}|}{n}$ free of $|{\cA}|$. We find any Kullback-Leibler divergence minimizer to be optimal in the last case. Numerical simulations distinguish the four learners and corroborate our theory.
\end{abstract}

\section{Introduction}\label{sec:intro}

It has become common sense that transferring intrinsic information from teachers to the greatest extent can expedite a student's learning progress, especially in machine learning given versatile and powerful teacher models. Learning with their assistance has been coined \emph{knowledge distillation} (KD) \citep{hinton2015distilling,lopez2015unifying}, a famous paradigm of knowledge transfer leading to remarkable empirical effectiveness in \emph{classification} tasks across various downstream applications \citep{gou2021knowledge, wang2021knowledge, gu2023knowledge}.
The term \emph{distillation} implies a belief that the inscrutable teacher(s) may possess useful yet complicated structural information, which we should be able to compress and inject into a compact one, i.e., the student model \citep{breiman1996born, buciluǎ2006model, li2014learning, ba2014deep,allen2020towards}.
This has guided the community towards a line of knowledge transfer methods featuring the awareness of teacher training details or snapshots, such as the original training set, the intermediate activations, the last-layer logits (for a probabilistic classifier), the first- or second-order derivative or statistical information, and even task-specific knowledge \citep{hinton2015distilling, furlanello2018born, cho2019efficacy, zhao2022decoupled, romero2014fitnets, zagoruyko2016paying, yim2017gift, huang2017like, park2019relational, tian2019contrastive, tung2019similarity, qiu2022better, srinivas2018knowledge, cheng2020explaining,liang2023less}.

However, in more general modes of knowledge transfer, the procedure- or architecture-specific information can be unavailable, irrelevant, or ill-defined. Modern proprietary large language models (LLMs) like GPT-4 \citep{openai2023gpt4} and Claude 2 \citep{anthropic2023claude2} strictly confine the data returned through their APIs except the generated tokens, not to mention their training sets. Therefore, there has been an effort towards distilling LLMs only through oracle queries via locally crafted prompts \citep{zheng2023judging, peng2023instruction}. In denoising \citep{tsybakov2003optimal}, boosting \citep{freund1995desicion}, and prediction with expert advice \citep{cesa1997use}, the teacher group itself is the classes of interest and the student only gets positive-valued feedback from the teacher group, so the teacher implementation is immaterial. In certain circumstances of robot learning \citep{abbeel2004apprenticeship, ross2010efficient}, the teacher demonstration may come from threshold-based classifiers or even human experts, where no teacher probability is defined.

To unify different natures of teacher information acquisition in knowledge transfer and assess the technical barriers in a principled way, we decompose the \emph{transfer set} ({the information a student has access to in total}) into a generative model $\pi^\star(\cdot|\cdot)$ giving one $a$ in the label space $\cA$ given a query $s$ in the input space $\cS$ and the additional information provided by the teacher for each sample pair $(s, a)$. This informative decomposition is motivated by a more general \emph{learning using privileged information} (LUPI) framework \citep{vapnik2015learning, lopez2015unifying}. Take several LLMs as examples. In OpenAI's chat completion API\footnote{\url{https://platform.openai.com/docs/api-reference/chat}} for GPT-4, only responses are returned given a prompt $s$, in which no privileged information is accessible. An early established version\footnote{\url{https://platform.openai.com/docs/api-reference/completions}} for GPT-3 \citep{brown2020language}, however, also returns the log-likelihood of every token in the generated response. If further with open-sourced models like GPT-2 \citep{radford2019language} in hand, practitioners can extract the last-layer logits of each position in the sampled sequence as the privileged information. These typical examples reveal a trend in the development of language models\footnote{See, e.g., \citet{ramesh2021zero, ramesh2022hierarchical, openai2023dall-e-3} in computer vision for a similar tendency.} that \emph{the more powerful a foundation model \citep{bommasani2021opportunities} is, the less likely it is to release privileged information as a teacher model}, which raises a grandiose question beyond learning from neural teacher networks:

\begin{center}
    \textit{What is the fundamental limit of knowledge transfer given limited privileged information in general?}
\end{center}

In this work, we answer this question with minimal inductive bias through the analyses of three protocols of information extraction from the teacher with easier difficulties in order in terms of the (offline) accessibility of the reference policy $\pi^\star$ over finite $\cS$ and $\cA$. Concretely, for any learner $\hat\pi$, we study the total variation between it and the reference policy $\pi^\star$ conditioned on an input distribution $\rho$, i.e., 
$$
\sum_{s\in\cS}\rho(s)\tv{\hat\pi(\cdot|s)}{\pi^\star(\cdot|s)} \eqqcolon \condtv{\hat\pi}{\pi^\star}{\rho},
$$
where $\hat\pi$ can utilize $n$ samples $\{(s_i, a_i)\}_{i=1}^n$ and optionally certain fraction of $\pi^\star(\cdot|s_i), \forall i\in[n]$.

\subsection{Main Contributions and Clarifications}

Table~\ref{tab:resultsTeaser} gives an overview of rate bounds in distinct data acquisition protocols. \dataone{} indicates $\pi^\star$ to be a black-box for $\hat\pi$ throughout the learning process. \datatwo{} stands for \emph{partial soft labels}, which means the teacher provides the student with an extra (partial) ground truth $\pi^\star(a|s)$ for every sample $(s, a)$. \datathree{} is a synonym of the logits on the entire $\cA$ exposed to $\hat\pi$ given each $s$ in $\{s_i\}_{i=1}^n$. Intuitively, each of the three levels discloses richer information than previous ones, which at least cannot make it harder to learn. Rigorously, progressively faster minimax rates are matching exactly at all levels\footnote{We defer the matching expectation bounds to respective sections.} and every high-probability upper bound has its upper confidence radius at most the same order of the corresponding expectation lower bound up to poly$\log$ factors. 

\begin{table}[htbp]
    \caption{Major theoretical findings of knowledge transfer with different acquired data. The second column shows one out of $n$ data points available to the student. Worst-case bounds in the last column hold with high probability and $\tilde{O}(\cdot)$ hides poly$\log$ factors.}
    \label{tab:resultsTeaser}
    \begin{center}
    \begin{tabular}{llc|cc}
    Transfer via & Data Format & Lower Bound & Loss & Performance \\\midrule
    \dataone{} & $(s, a)$ & \makecell[c]{$\Omega\left(\sqrt{\nicefrac{|\cS||\cA|}{n}}\right)$ \\ Theorem~\ref{thm:lb1}} & $\lossCEsingle$ & \makecell{$\tilde{O}\left(\sqrt{\nicefrac{|\cS||\cA|}{n}}\right)$ \\ Theorem~\ref{thm:mle}} \\\midrule
    \multirow[b]{2}{*}{\datatwo{}} & \multirow[b]{2}{*}{$(s, a, \pi^\star(a|s))$} & \multirow[b]{2}{*}{\makecell{$\Omega\left(\nicefrac{|\cS||\cA|}{n}\right)$ \\ Theorem~\ref{thm:lb2}}} & $\lossEMPCE$ & \makecell{$\Omega(1)$ \\ Theorem~\ref{thm:empce_bias}} \\
    & & & $\lossEMPSEL$ & \makecell{$\tilde{O}\left(\nicefrac{|\cS||\cA|}{n}\right)$ \\ Theorem~\ref{thm:empselRefined}} \\\midrule
    \datathree{} & $(s, a, \pi^\star(\cdot|s))$ & \makecell[c]{$\Omega\left(\nicefrac{|\cS|}{n}\right)$ \\ Theorem~\ref{thm:lb3}} & $\lossCEfull$ & \makecell{$\tilde{O}\left(\nicefrac{|\cS|}{n}\right)$ \\ Theorem~\ref{thm:full}}
    \end{tabular}
    \end{center}
\end{table}

Technically, the later two protocols are nonstandard in density estimation, especially in terms of the derivation of minimax lower bounds in that we assume the data generating process to be $(\rho\times\pi^\star)^n$. The constructive proof (Appendix~\ref{sec:appendix:lb2proof}) of Theorem~\ref{thm:lb2} may be of independent interest. The performances at the second level are also more tricky. Theorem~\ref{thm:empce_bias} sentences a naive adaptation of the cross-entropy loss ($\lossEMPCE$) to be a misfit with probability $1$ when the privileged information is at a modest level. 

We do not require $\pi^\star$ to be trained in any sense and assume $\hat\pi$ to have no awareness of the teacher implementation, inspired by the aforementioned practical trend. Consequently, we do no distinguish the \emph{teacher probability} from the \emph{Bayes probability} \citep{menon2020distillation, dao2021knowledge} and $\rho$ can have no relation with teacher training. We idealize the samples to be \emph{i.i.d.}, while all the results already extend to the setting where different $(s,a)$ pairs are mildly dependent, of which we refer the readers to Appendix~\ref{sec:extToTransKer} for detailed discussions.

We give more discussions on the connections between our framework and previous formulations in Appendix~\ref{subsec:connections} and review key related works aiming at the principles of knowledge transfer in the next subsection.

\subsection{Related Works on Understanding Knowledge Transfer}

\citet{hinton2015distilling} refers to the logits generated on a dataset by a teacher model, who has been trained using the same dataset, as soft labels; and refers to $\{a_i\}_{i=1}^n$ in this dataset as hard labels. Our terms, however, pay no attention to whether $\rho$ matches the original dataset; and our \datathree{} strictly carry more information (of $\rho\times\pi^\star$) than our \dataone{}, both of whose inputs $\{s_i\}_{i=1}^n$ are sampled from $\rho$. So the third column of Table~\ref{tab:resultsTeaser} does not account for the \emph{class similarities} argument \citep{furlanello2018born}, the \emph{regularization effect via label smoothing (LS)} argument \citep{yuan2020revisiting, tang2020understanding}, or the \emph{bias-variance trade-off} argument \citep{zhou2021rethinking, menon2020distillation}; all of which are classical wisdom explaining the benefit of soft labels in KD \citep{hinton2015distilling}.

These previous views are neither consistent nor complete for justifying why soft labels improves student training. \citep{muller2020subclass} undermines the hypothesis that class similarities in soft labels are vital by the effectiveness of KD in binary classification. \citet{han2023impact} challenges the regularization via LS thesis through better interpretablity-lifting effect of KD than LS. \citet{dao2021knowledge} develops the bias-variance trade-off perspective in more complex scenarios. In contrast, we define \datathree{} (similarly \datatwo{}) and tackle its boost over \dataone{} rigorously following the information-theoretic direction of the data processing inequality \citep{polyanskiy2022information}. We survey more works unraveling knowledge transfer. Other empirical and enlightening paradigms are deferred to Appendix~\ref{sec:appendix:additionalRW}.

Most analyses of KD with trained deep nets as the teacher \citep{phuong2019towards, ji2020knowledge, panahi2022analysis, harutyunyan2023supervision} lie in the linear or kernel regime, notably except \citet{hsu2021generalization}, which finds the student network to have fundamentally tighter generalization bound than the teacher network under several nonlinear function approximation schemes. There are also works analyzing knowledge transfer between neural networks of the same architecture \citep{mobahi2020self, allen2020towards}. Our framework goes exactly in the reverse direction: our analysis is not restricted to parsimonious students, over-parameterized teachers; or the \emph{compression} subconcept of knowledge transfer \citep{buciluǎ2006model, bu2020information}. For example, a human demonstrator, who does not learn only from data and is able to output probabilistic belief, can also fit into the first column of Table~\ref{tab:resultsTeaser} as a kind of teacher under our specification of $\pi^\star$.

\section{Preliminaries}

\paragraph{Notation.} For two nonnegative sequences $\{a_n\}$ and $\{b_n\}$, we write $a_n = O(b_n)$ or $a_n \lesssim b_n$ if $\lim\sup \nicefrac{a_n}{b_n} < \infty$; equivalently, $b_n = \Omega(a_n)$ or $b_n \gtrsim a_n$. We write $c_n = \Theta(d_n)$ or $c_n \asymp d_n$ if $c_n = O(d_n)$ and $c_n = \Omega(d_n)$. For a metric space $(\cM, d)$, $\dist{d}{\cdot}{\cN} \coloneqq \inf_{y\in\cN} d(\cdot, y)$ for any $\cN \subset \cM$. The term \emph{alphabet} is a synonym of \emph{finite set}. For a set or multiset $\cC$, let $|\cC|$ denote its cardinality. For any alphabet $\cX$, on which given two distributions $p$, $q$, the \emph{total variation} between them is $\tv{p}{q} \coloneqq 0.5\norm{p - q}_1$, their \emph{Kullback-Leibler (KL) divergence} is $\kl{p}{q} \coloneqq \EE_p[\log p - \log q]$; and we denote all $|\cX|$ Dirac distributions on $\cX$ by $\mathsf{Dirac}(\cX)$, where $\mathsf{Dirac}(\cX, x)$ stands for the one concentrated at $x$.
For two alphabets $\cX$ and $\cY$, we denote by $\Delta(\cY)$ the probability simplex on $\cY$ and define $\Delta(\cY|\cX) \coloneqq \{\pi(\cdot|\cdot): \cX \to \Delta(\cY)\}\Rightarrow \pi(\cdot|x)\in\Delta(\cY), \forall x\in\cX$.

\subsection{Common Setup}

The teacher always exposes to the student a multiset $\cD = \{(s_i, a_i)\in\cS\times \cA\}_{i=1}^n$ consisting of $n$ \emph{i.i.d.} input-label tuples. To analyze $\cD$ in a fine-grained way we introduce for $\cX^n\ni X^n\stackrel{i.i.d.}{\sim}\nu$ the \emph{number of occurrences} $\nm{x}{X^n}$ of $x$ and 
 the \emph{missing mass} $\mm{\nu}{X^n}$, which measures the portion of $\cX$ never observed in $X^n$ \citep{mcallester2003concentration}. We refer to the input (resp. label) component $\{s_i\}_{i=1}^n$ (resp. $\{a_i\}_{i=1}^n$) of $\cD$ as $\cS(\cD)$ (resp. $\cA(\cD)$) and also define a multiset $\cA(\cD, s)$ for every $s\in\cS$ to denote the $a_i$'s in $\cD$ that are associated with the visitations of $s$, taking into account multiplicity.\footnote{See Appendix~\ref{sec:appendix:missing} for the rigorous definitions of $\nm{x}{X^n}$, $\mm{\nu}{X^n}$, and $\cA(\cD, s)$.}

\subsection{Quantity of Interest}

We assume $\cD\sim(\rho\times\pi^\star)^n$ follows a product measure, where $\pi^\star$ is the ground truth distribution over $\cA$ given every $s\in\cS$ 
and $\rho$ is some underlying input generating distribution. No assumption is imposed on the data generating process $\rho\times\pi^\star$ except for belonging to 
\begin{equation*}
    \cP \coloneqq \left\{\rho \times \pi^\star: \rho\in\Delta(\cS), \pi^\star\in\Delta(\cA|\cS) \right\}.
\end{equation*}

In this work, we evaluate the performance of a student $\hat\pi$ based on the $\mathsf{TV}$ between $\hat\pi$ and $\pi^\star$ conditioned on $\rho$ defined as
\begin{equation}\label{eq:qoi}
    \condtv{\hat\pi}{\pi^\star}{\rho} \coloneqq \EE_{s\sim\rho}\left[\tv{\hat\pi}{\pi^\star}\right],
\end{equation}
though the student is never allowed to access $\rho$ directly. We investigate the convergence rate of \eqref{eq:qoi} among three categories of students told by the teacher's degree of openness in the tabular setting. Intuitively speaking, all the learning procedures of interest try to match the log-probability kernel $\log\pi$, a notion of \emph{normalized} {logits}, between the student and the teacher, especially via variants of the cross-entropy loss, which is standard in the study of classification both theoretically and practically \citep{Paszke_PyTorch_An_Imperative_2019}. Besides the universal definition $\lossCEfull(p\Vert q) \coloneqq -\EE_p[\log q]$ of the cross-entropy between $p\ll q$, a popular counterpart for hard labels specialized to classifiers is commonly defined as $\lossCEsingle(s, a; \pi) \coloneqq \lossCEfull\left(\mathsf{Dirac}(\cA, a)\Vert \pi(\cdot|s)\right) = -\log\pi(a|s)$.

\section{Transfer via \dataone{}}

This is equivalent to the standard setting for estimating the conditional density $\pi^\star(\cdot|\cdot)$.

\subsection{Hardness of Estimation}

We first generalize the idea of constructing hard instances for learning discrete distributions on $\cA$ \citep{paninski2008coincidence} to our nonsingleton $\cS$ to understand the difficulty when only $(s,a)$ paris are available. We remark that the proof of Theorem~\ref{thm:lb1} (in Appendix~\ref{sec:appendix:lb1proof}) is the only one in this work that utilizes Assouad's method \citep{yu1997assouad} directly.

\begin{theorem}\label{thm:lb1}
    For nonempty $\cS$, $\cA$ with $|\cA| > 1$, and $n \geq \nicefrac{|\cS||\cA|}{4}$,
    \begin{equation}\label{eq:lbSentence}
        \inf_{\hat\pi \in \hat\Pi(\cD)} \sup_{\rho\times \pi^\star \in \cP} \EE_{(\rho\times \pi^\star)^n}\condtv{\hat\pi}{\pi^\star}{\rho} \gtrsim \sqrt{\frac{|\cS||\cA|}{n}},
    \end{equation}
    where $\cD\sim (\rho\times \pi^\star)^n$, $\hat\Pi(\cD)$ denotes all (possibly random) estimators mapping $\cD$ to $\Delta(\cA|\cS)$.
\end{theorem}

The $\sqrt{|\cS|}$ dependence in this lower bound intuitively makes sense because the classic lower bound for $|\cS| = 1$ is $\Omega(\sqrt{\nicefrac{|\cA|}{n}})$ and each input roughly get $\nicefrac{n}{|\cS|}$ samples when $\rho$ is distributed evenly.

\subsection{Maximum Likelihood Estimation}

We approximate the teacher $\pi^\star$ via minimizing the following negative log-likelihood loss:

\begin{equation}\label{eq:def:lossCEsingle_D}
    \nll \in \argmin_{\pi\in\Delta(\cA|\cS)} \lossCEsingle(\cD) \coloneqq \argmin_{\pi\in\Delta(\cA|\cS)} -\sum_{i=1}^n \log \pi(a_i| s_i).
\end{equation}

It is possible to exactly attain the minimum $0$ in \eqref{eq:def:lossCEsingle_D}. A refactoring detailed in Appendix~\ref{subsec:appendix:nllSolDerive} indicates a natural relation between the hard version $\lossCEsingle$ and the soft version $\lossCEfull$, which leads to a neat closed-form solution for the optimization problem.

\begin{equation}\label{eq:mleSol}
    \nll(a | s) \begin{cases}
        = \nm{(s,a)}{\cD} / \nm{s}{\cS(\cD)}, &s \in \cS(\cD),\\
        \in\Delta(\cA)\text{ arbitrarily}, &\text{otherwise.}
    \end{cases}
\end{equation}

We study the convergence behavior of $\nll$ in a fine-grained way with its proof detailed in Appendix~\ref{subsec:mleProofs}:

\begin{theorem}
\label{thm:mle} For any $\delta \in (0, 1)$, with probability at least $1 - \delta$,
    \begin{equation}\label{eq:mle:highProbBound}
        \condtv{\nll}{\pi^\star}{\rho} \lesssim \sqrt{\frac{|\cS|\left(|\cA| + \log\left(|\cS|/\delta\right)\right)}{n}}.
    \end{equation}
    The upper bound in expectation $\EE\left[ \condtv{\nll}{\pi^\star}{\rho} \right] \lesssim \sqrt{\nicefrac{|\cS||\cA|}{n}}$ is no better than \eqref{eq:mle:highProbBound} up to $\log$ factors. An instance-dependent version in expectation is
    \begin{equation}\label{eq:mle:instanceDependent}
        \EE\left[ \condtv{\nll}{\pi^\star}{\rho} \right] \lesssim \sqrt{\xi(\pi^\star)\frac{|\cS||\cA|}{n}} + \frac{|\cS|}{n},
    \end{equation}
    where $\xi(\pi^\star)\coloneqq \max_{s\in\cS}\dist{\mathsf{TV}}{\pi^\star(\cdot|s)}{\mathsf{Dirac}(\cA)} = 2\max_{s\in\cS}\min_{a\in\cA}(1 - \pi^\star(a|s))$.
\end{theorem}

Thus, $\nll$ is worst-case optimal and may even have a $n^{-1}$ rate in some benign cases with $\pi^\star$ close enough to vertices of the simplex $\Delta(\cA)$.

\section{Transfer via \datatwo{}}

Besides $\cD$, we can also access $\cR \coloneqq \{(s_i, a_i, \pi^\star(a_i | s_i))\}_{i=1}^n$ as \datatwo{}. The introduction of $\cR$ leads to a quadratic reduction of the learning difficulty.

\subsection{Blessing of Ground Truth}

\begin{theorem}\label{thm:lb2}
    For nonempty $\cS$, $\cA$ with $|\cA| > 2$, and $n \geq \nicefrac{|\cS|(|\cA| - 1)}{2} - 1$,
    \begin{equation}
        \inf_{\hat\pi \in \hat\Pi(\cD, \cR)} \sup_{\rho\times \pi^\star \in \cP} \EE_{(\rho\times \pi^\star)^n}\condtv{\hat\pi}{\pi^\star}{\rho} \gtrsim \frac{|\cS||\cA|}{n},
    \end{equation}
    where $\cD\sim (\rho\times \pi^\star)^n$, $\cR = \left\{\left(s, a, \pi^\star(a|s)\right): (s, a)\in\cD\right\}$, and $\hat\Pi(\cD, \cR)$ denotes all (possibly random) learners mapping $(\cD, \cR)$ to $\Delta(\cA|\cS)$.
\end{theorem}

We provide a constructive proof of Theorem~\ref{thm:lb2} in Appendix~\ref{sec:appendix:lb2proof}, in which we resort to the power of randomized policy so as to reveal the linear in $|\cA|$ dependence.

\subsection{Failure of Empirical Cross-Entropy Loss}

The \datatwo{} $\cR$ motivates us to define a loss $\lossEMPCE$ interpolating between $\lossCEsingle$ and $\lossCEfull$.
\begin{equation}\label{eq:def:lossEMPCE_D+R}
    \empce \in \argmin_{\pi\in\Delta(\cA|\cS)}\lossEMPCE(\cD, \cR) \coloneqq \argmin_{\pi\in\Delta(\cA|\cS)} -\sum_{i=1}^n  \pi^\star(a_i | s_i) \log\pi(a_i| s_i).
\end{equation}
We can obtain the following exact solution to \eqref{eq:def:lossEMPCE_D+R} by another technique detailed in Appendix~\ref{subsec:appendix:empceSolDerive}.
\begin{equation}\label{eq:empceSol}
    \empce(a | s) \begin{cases}
        \propto \nm{(s, a)}{\cD}\pi^\star(a|s), &s\in\cS(\cD),\\
        \in\Delta(\cA)\text{ arbitrarily}, &s\notin\cS(\cD).
    \end{cases}
\end{equation}

The convergence analysis of $\empce$ crucially relies on its relationship with $\nll$. For any $s\in\cS(\cD)$, $\empce$ can be reformulated as

\begin{equation}\label{eq:empce_rel_mle}
    \empce(a|s) = \frac{\pi^\star(a|s)\nll(a|s)}{\sum_{b\in\cA}\pi^\star(b|s)\nll(b|s)}.
\end{equation}

\begin{lemma}\label{lem:empce_bias}
    For $|\cS| = 1$, 
    $$
    \empce(\cdot|s) \stackrel{a.s.}{\longrightarrow} \left[\pi^\star(\cdot|s)\right]^2 / \sum_{a\in\cA}\left[\pi^\star(a|s)\right]^2,
    $$
    under $\ell_\infty$ in $\RR^{|\cA|}$ if $\rho\times\pi^\star$ is independent of $n$.
\end{lemma}

Lemma~\ref{lem:empce_bias} roughly means $\empce \propto (\pi^\star)^2$ approximately as $n$ goes to infinity, which implies that small parts of $\pi^\star$ are underestimated and large parts of $\pi^\star$ are overestimated. We make this intuition technically right in Theorem~\ref{thm:empce_bias}, whose rigorous statement, mechanism, and proof are deferred to Appendix~\ref{subsec:Prooff:thm:empce_bias}.

\begin{theorem}\label{thm:empce_bias}
    If $\rho\times\pi^\star$ does not vary with $n$, $\empce$ coincides with $\nll$ (and thus asymptotically unbiased\footnote{Though $\empce$ is not an estimator, we can discuss unbiasedness under a more general notion, i.e., for $K$ constants $\{c_i\}_{i=1}^K$, the random variables $\{X_{i, n}\}_{i=1}^K$ is called asymptotically unbiased if $X_{i, n}\to c_i$ in some mode of convergence as $n\to\infty$ for every $i\in[K]$.}) only if $\pi^\star(\cdot|s) = \mathsf{Uniform}(\cA)$ or $\pi^\star(\cdot|s)\in\mathsf{Dirac}(\cA)$ for all $s\in\cS$. Even for $|\cS| = 1$, $\empce$ is asymptotically biased in general.
\end{theorem}

\subsection{Empirical Squared Error Logit Loss}

\citet{ba2014deep} suggests the practically promising $\mathsf{SEL}$ loss\footnote{Practical versions of $\mathsf{SEL}$ often allow unnormalized logits. See Remark~\ref{rmk:2ndCase:norm_vs_unnorm} for more discussions.}:

\begin{equation}\label{eq:selLoss}
    \cL(\pi, \pi^\star) = \sum_{i=1}^n \frac{1}{2}\sum_{a\in\cA} \left[\log\pi(a| s_i) - \log\pi^\star(a| s_i)\right]^2.
\end{equation}

Here we analyze the minimization of its empirical variant with \emph{normalized} logits under the second data acquisition protocol for simplicity:
\begin{equation}\label{eq:def:lossEMPSEL}
    \lossEMPSEL(\cD, \cR) \coloneqq \sum_{i=1}^n \frac{1}{2} \left[\log\pi(a_i| s_i) - \log\pi^\star(a_i| s_i)\right]^2.
\end{equation}
Exact matching on the seen samples in $(\cD, \cR)$ shows that $\empsel \in \argmin_{\pi\in\Delta(\cA|\cS)} \lossEMPSEL(\cD, \cR)$
\begin{equation}\label{eq:empselSol}
 \Leftrightarrow \empsel(\cdot| s)\in \begin{cases}
    \{p\in\Delta(\cA): p(a) = \pi^\star(a|s), \forall a\in\cA(\cD, s)\}, &s \in \cS(\cD),\\
        \Delta(\cA)\text{ arbitrarily}, &\text{otherwise.}
    \end{cases}
\end{equation}

The following three-fold Theorem~\ref{thm:empselRefined} with its proof in Appendix~\ref{subsec:Prooff:thm:empce_bias} indicates that $\empsel$ converges faster than $\nll$ by a factor of $\sqrt{n}$ though its performance upper bound has worse dependence on $|\cS||\cA|$ compared with that of $\nll$.

\begin{theorem}\label{thm:empselRefined}
    If $|\cS| > 1$, for any $\delta\in (0, \min(1, \nicefrac{(|\cS| + 2)}{10}))$, with probability at least $1 - \delta$,
    \begin{equation}\label{eq:empselRefined:sentence}
        \condtv{\empsel}{\pi^\star}{\rho} \lesssim \frac{|\cS|\left(|\cA| + \sqrt{|\cA|\log\left(|\cS|/\delta\right)}\right)}{n}\log\frac{|\cS|}{\delta}.
    \end{equation}
    If $|\cS| = 1$, for any $\delta\in(0, \nicefrac{1}{10}]$, with probability at least $1 - \delta$,
    \begin{equation}\label{eq:empselRefined:S=1}
        \condtv{\empsel}{\pi^\star}{\rho} = \tv{\empsel}{\pi^\star} \lesssim \frac{|\cA|}{n} + \frac{\sqrt{|\cA|}}{n}\log\frac{1}{\delta}.
    \end{equation}
    The expected risk $\EE\condtv{\empsel}{\pi^\star}{\rho} \lesssim \nicefrac{|\cS||\cA|}{n}$ 
    is not polynomially tighter than \eqref{eq:empselRefined:sentence} or \eqref{eq:empselRefined:S=1}.
\end{theorem}

\begin{remark}
    Theorem~\ref{thm:empce_bias} together with Theorem~\ref{thm:empselRefined} manifests the advantage of employing the empirical SEL loss, which induces an alignment between the normalized logits of the learner and those of $\pi^\star$ under squared loss, over the empirical CE loss in offline distillation when the teacher is moderately reserved. A similar observation between these two style of empirical surrogate losses in online policy optimization is verfied in practice \citep{zhu2023fine}.
\end{remark}

\section{Transfer via \datathree{}}\label{sec:transferViaDatathree}

At the lightest level, the student has the extra information $\cQ = \{(s, \pi^\star(\cdot|s):s\in\cS(\cD)\}$. The availability of $\cQ$ apparently eases the transfer process, especially when the support size $|\cA|$ of the teacher classifier is huge. Such an intuition can be precisely depicted by a $|\cA|$-free minimax lower bound.

\subsection{$|\cA|$-Free Lower Bound}

The following lower bound roughly requires $\Omega(|\cS|)$ burn-in cost, whose constructive proof is deferred to Appendix~\ref{sec:appendix:lb3proof}.

\begin{theorem}\label{thm:lb3}
    For $\cS$ with $|\cS| > 1$, $\cA$ with $|\cA| > 1$, and $n > |\cS| - 1$,
    \begin{equation}\label{eq:lb3:Sentence}
        \inf_{\hat\pi \in \hat\Pi(\cD, \cQ)} \sup_{\rho\times \pi^\star \in \cP} \EE_{(\rho\times \pi^\star)^n}\condtv{\hat\pi}{\pi^\star}{\rho} \gtrsim {\frac{|\cS|}{n}},
    \end{equation}
    where $\cD\sim (\rho\times \pi^\star)^n$, $\cQ = \left\{\left(s, \pi^\star(\cdot|s)\right): s\in\cS(\cD)\right\}$, and $\hat\Pi(\cD, \cQ)$ denotes all (possibly random) learners mapping $(\cD, \cQ)$ to $\Delta(\cA|\cS)$.
\end{theorem}

This setting cannot have a rate better than $n^{-1}$, which is consistent with the $n^{-1}$ rate in Theorem~\ref{thm:lb2} since the difficulties of the later two settings (intuitively, the information provided at these two levels) should be the same when $\pi^\star(\cdot|s)\in\mathsf{Dirac}(\cA)$ for any $s\in\cS$.

\subsection{Kullback-Leibler Divergence Minimization}

Cross-entropy loss minimization under full observation is equivalent to
\begin{equation}\label{eq:CE_full_def}
    \fullkl \in \argmin_{\pi} \sum_{i=1}^n \kl{\pi^\star(\cdot | s_i)}{\pi(\cdot | s_i)}\Rightarrow \fullkl(\cdot| s)  \begin{cases}
        =\pi^\star(\cdot| s), &\text{if }s \in \cS(\cD),\\
        \in\Delta(\cA)\text{ arbitrarily}, &\text{otherwise.}
    \end{cases}
\end{equation}

We give the missing-mass-based proof of the matching upper bounds for $\fullkl$ in Theorem~\ref{thm:full} in Appendix~\ref{subsec:fullklProofs}.

\begin{theorem}\label{thm:full}
    For any $\delta \in (0, \nicefrac{1}{10}]$, with probability at least $1 - \delta$,
    \begin{equation}\label{eq:fullHighProbBound}
        \condtv{\fullkl}{\pi^\star}{\rho} \lesssim \frac{|\cS|}{n} + \frac{\sqrt{|\cS|}}{n}\log\frac{1}{\delta}.
    \end{equation}
    The upper bound $\EE\condtv{\fullkl}{\pi^\star}{\rho} \lesssim \nicefrac{|\cS|}{n}$ on the expected risk for $\fullkl$ nearly matches \eqref{eq:fullHighProbBound}.
\end{theorem}

Theorem~\ref{thm:full} guarantees the optimality (not only in expectation but also in expectation) of $\fullkl$ in that it maximally utilizes the given logits.

\section{Experiments}

We conduct simulations 
to verify the intuitive performance rankings $\nll\preceq\empsel\preceq\fullkl$ given moderately large sample sizes and also numerically provide the asymptotical biasedness of $\empce$ with a finite-sample counterpart. Moreover, we design adversarial data generating distributions inspired by the information-theoretic arguments (Appendix~\ref{sec:appendix:lbs}) for the three types of reserved teachers 
respectively in the non-asymptotic regime, 
thereby accurately exhibiting 
the matching convergence rates of $\nll$, $\empsel$, and $\fullkl$ in terms of $n$.

In this section, we specify a fair inductive bias due to the tabular nature: if $s\notin \cS(\cD)$, $\hat\pi(\cdot|s)$ is set to $\mathsf{Uniform}(\cA)$ for all learners; for $\empsel(\cdot|s)$, the missing mass is amortized uniformly among $\cA\backslash\cA(\cD, s)$ if $s\in\cS(\cD)$.

\subsection{Classic Regime: Telling Learners Apart}

In the classic regime, $\rho\times\pi^\star$ stays invariant no matter whether $n$ tends to infinity or not. An instance in this sense should not only expose the inferior of $\empce$ but also showcase the hardness of $|\cS| > 1$, i.e., $\rho$ should be strictly bounded away from zero in $\Omega(|\cS|)$ inputs. To these ends, we specify \emph{Instance~\hyperref[instance:0]{0}} in Appendix~\ref{appendix:ExperimentsSepc:subsec_instances}. We simulate four typical realizations of it, whose estimated risks is presented in Figure~\ref{fig:tabular_exp}. Each marker in Figure~\ref{fig:tabular_exp} represents the empirical mean of $\condtv{\hat\pi}{\pi^\star}{\rho}$ in 100 independent repeats given the corresponding sample size $n$. Any broken line in Figure~\ref{fig:tabular_exp} has nothing to do with sequential design and our experiment is purely offline. As shown in Figure~\ref{fig:tabular_exp} (a, b), either in a general case with $(|\cS|, |\cA|)=(100, 25)$ or for a 100-armed rewardless bandit, $\hat\EE\condtv{\empce}{\pi^\star}{\rho}$ fails to converge but all other learners do, corroborating the asymptotically constant bias of $\empce$ and the consistency of the others. 

\begin{figure}[htbp]
    \centering
    \includegraphics[width=\textwidth]{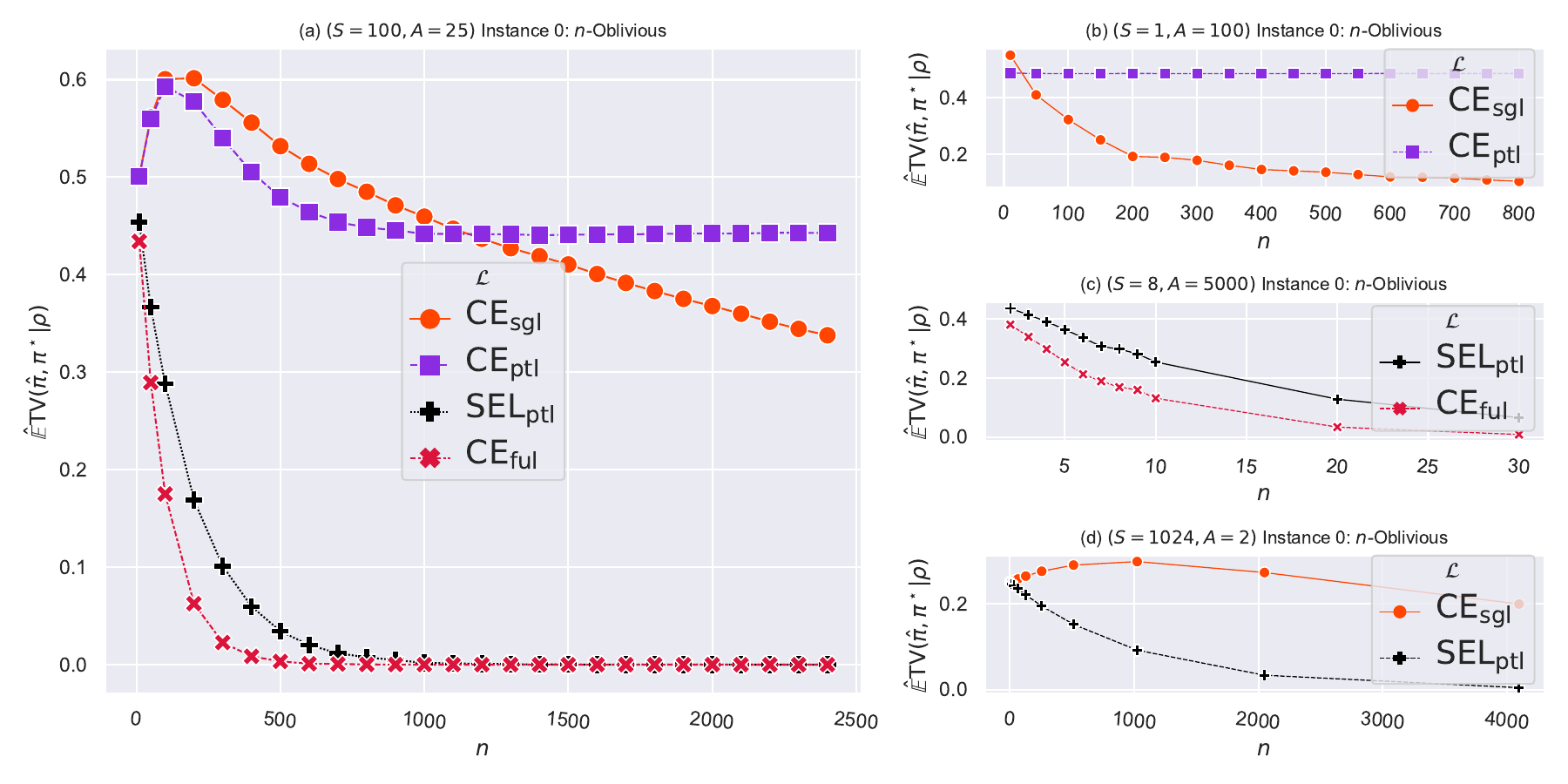}
    \caption{Estimated risks in expectation. $\rho\times\pi^\star$ does not vary with $n$.}
    \label{fig:tabular_exp}
\end{figure}

Though Instance~\hyperref[instance:0]{0} is effectively so easy to learn for any of $\nll$, $\empsel$, and $\fullkl$ that none of the three worst-case upper bounds is tightly attained, the numerical performance rankings among them in Figure~\ref{fig:tabular_exp} coincide with our intuition and theoretical analysis. The ``benign''-case comparison in Figure~\ref{fig:tabular_exp} (c), where the sample sizes are small enough to make the worst-case lower bounds vacuous, still favors $\fullkl$ over $\empsel$ in the large-$|\cA|$ regime. Figure~\ref{fig:tabular_exp} (d) manifests that the $\sqrt{|\cS||\cA|}$ gap between the worst-case upper bounds for $\nll$ and $\empsel$, which is reversely dominated by the exponential rate\footnote{See Remark~\ref{rmk:expDecay} for details.} of $\empsel$ in this Instance~\hyperref[instance:0]{0}, may not be observed in general even for large $|\cS||\cA|$ and small $n$.

\begin{remark}\label{rmk:expDecay}
    Direct calculations imply that if $\rho \geq c_\cS > 0$ for all inputs with $c_\cS$ irrespective of $n$ when $n$ is large enough, $\EE\condtv{\fullkl}{\pi^\star}{\rho}$ can decay exponentially fast, exemplified by Figure~\ref{fig:tabular_exp} (a). $\EE\condtv{\empsel}{\pi^\star}{\rho}$ will enjoy a similar linear convergence so long as we additionally require $\pi^\star(\cdot|s) \geq c_\cA > 0$ for all $s\in\cS$ and all labels with $c_\cA$ independent of $n$ for sufficiently large $n$, again exemplified by Figure~\ref{fig:tabular_exp} (a).
\end{remark}

\subsection{Non-Asymptotic Regime: Illustration of Matching Rates}

Instance~\hyperref[instance:0]{0} serves as an intriguing average case, but we need to design worst-case instances that may vary with $n$ \citep{wainwright2019high} in the non-asymptotic regime for different data acquisition settings in order to verify the minimax optimalities. The adversarial Instance~\hyperref[instance:1]{1},~\hyperref[instance:2]{2}, and~\hyperref[instance:3]{3} with their design insights are detailed in order in Appendix~\ref{appendix:ExperimentsSepc:subsec_instances} for verification of matching rates at all three levels.

Since $\nll$, $\empsel$, and $\fullkl$ enjoy optimal rates of order $n^{-1}$ or $n^{-0.5}$, we can manifest them using lines in a log-log plot. More generally, if some notion of $\texttt{risk}$ has $\texttt{risk} = \Theta(n^{\beta^\star})$ for some $\beta^\star < 0$, $\log\texttt{risk} - \beta^\star\log n$ will be at least bounded by two straight lines on a log-log scale. We instantiate the above idea for Instance~\hyperref[instance:1]{1}, Instance~\hyperref[instance:2]{2}, and Instance~\hyperref[instance:3]{3} in Figure~\ref{fig:adv_all_risks}, in which each marker represents the average of 64000 independent repeats. We also conduct linear regressions over $\log\texttt{risk}\sim\log n$ for corresponding minimax learners and report the slope $\hat\beta$ as estimated $\beta^\star$ in each subfigure of Figure~\ref{fig:adv_all_risks}.

\begin{figure}[htbp]
    \centering
    \includegraphics[width=\textwidth]{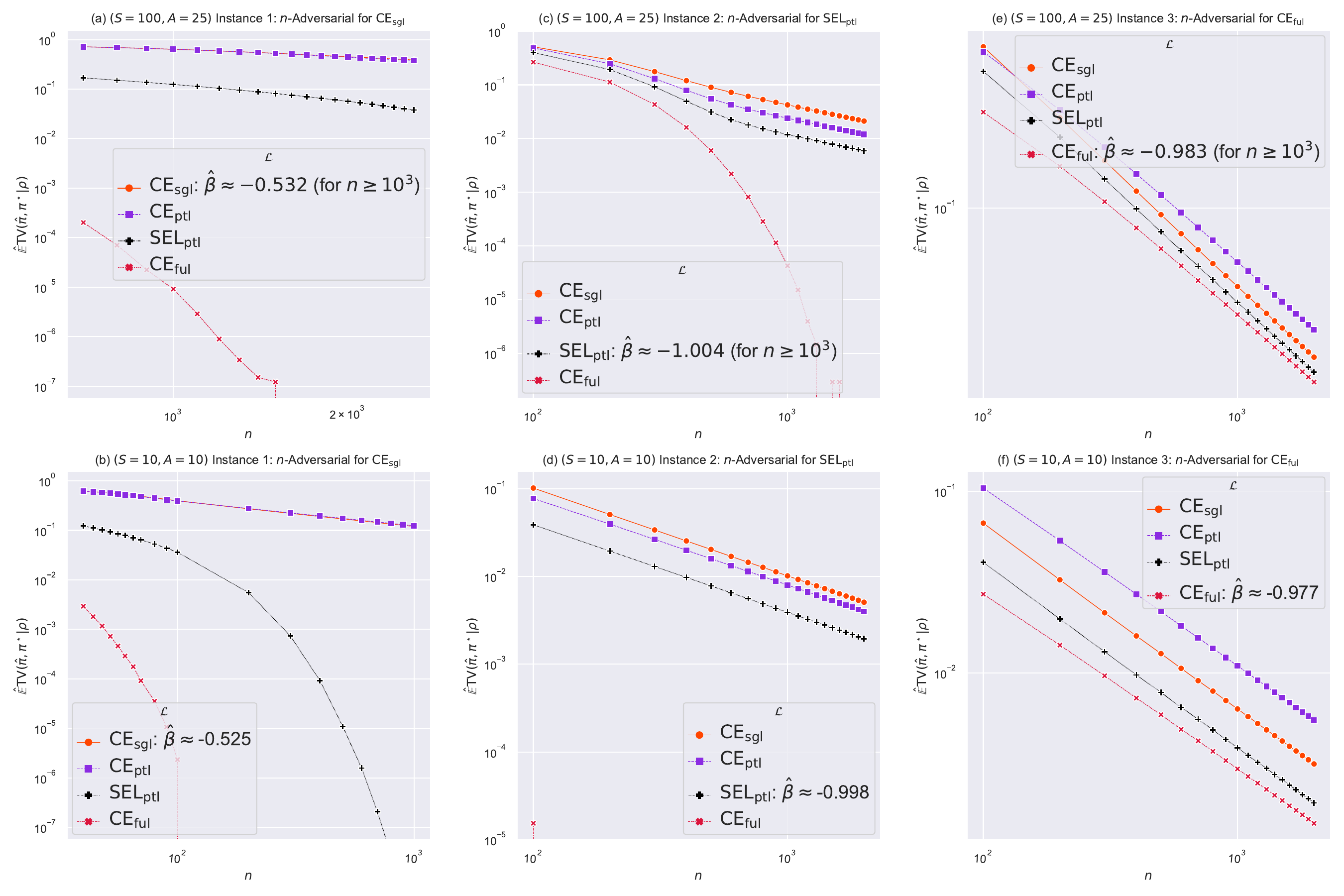}
    \caption{Estimated risks in expectation. $\rho\times\pi^\star$ scales with $n$ in setting-specific ways.}
    \label{fig:adv_all_risks}
\end{figure}

We elaborate on several interesting phenomena of Figure~\ref{fig:adv_all_risks}.

\begin{itemize}
    \item Regarding parameters other than $n$ as constants, the $\hat\beta$'s precisely verify the worst-case rates $\EE\condtv{\nll}{\pi^\star}{\rho}\asymp n^{-0.5}$ (Figure~\ref{fig:adv_all_risks} (a, b)), $\EE\condtv{\empsel}{\pi^\star}{\rho}\asymp n^{-1}$ (Figure~\ref{fig:adv_all_risks} (c, d)), and $\EE\condtv{\fullkl}{\pi^\star}{\rho}\asymp n^{-1}$ (Figure~\ref{fig:adv_all_risks} (e, f)).
    \item All Instance~\hyperref[instance:1]{1}, \hyperref[instance:2]{2}, and~\hyperref[instance:3]{3} happen to be too easy for $\empce$ to be obviously biased. This phenomenon is actually predictable because $\pi^\star(\cdot|s)$ becomes very close to $\mathsf{Uniform}(\cA)$ in Instance~\hyperref[instance:1]{1}, and to some one in $\mathsf{Dirac}(\cA)$ in both Instance~\hyperref[instance:2]{2} or~\hyperref[instance:3]{3} for all inputs when $n$ is relatively large. In these cases, $\empce$ largely coincides with $\nll$ as predicted by Theorem~\ref{thm:empce_bias}.
    \item The risks of $\empsel$ in Instance~\hyperref[instance:1]{1} (Figure~\ref{fig:adv_all_risks} (a, b)) and the risks of $\fullkl$ in Instance~\hyperref[instance:1]{1} (Figure~\ref{fig:adv_all_risks} (a, b)) and Instance~\hyperref[instance:2]{2} (Figure~\ref{fig:adv_all_risks} (c, d)) decay exponentially fast, which is consistent with the arguemnts in Remark~\ref{rmk:expDecay}, i.e., only vanishing schemes like $\pi^\star$ in Instance~\hyperref[instance:2]{2} for $\empsel$ and $\rho$ in Instance~\hyperref[instance:3]{3} for $\fullkl$ can force their risks to have a polynomial decay.
    \item We are able to provably explain the good performance (beyond worst cases) of $\nll$ in Instance~\hyperref[instance:2]{2} (Figure~\ref{fig:adv_all_risks} (c, d)) and Instance~\hyperref[instance:3]{3} (Figure~\ref{fig:adv_all_risks} (e, f)), in which $\xi(\pi^\star) \lesssim n^{-1}$. ({Recall the definition of $\xi(\cdot)$ in Theorem~\ref{thm:mle}.}) Thus, the instance-dependent bound in Theorem~\ref{thm:mle} indicates a benign-case rate of $\EE\condtv{\nll}{\pi^\star}{\rho} \lesssim n^{-1}$. 
\end{itemize}

\section{Conclusion}

We embark on investigating knowledge transfer beyond pure black- or white-box regimes and settle its sample complexity, respectively; provided that the teacher can afford (1) only to act as a generative model, or (2) additionally the probabilities at sampled classes, or (3) additionally the logits conditioned on each sampled input. The theoretical analysis unveils a crucial insight that tailoring the idea of minimizing $\mathsf{CE}$ to new information acquisition scenarios may be sub-optimal in general, provably in knowledge transfer via \datatwo{}. Several avenues remain to be further explored.

\begin{itemize}
    \item Huge $\cS$ and $\cA$ in practice \citep{zeng2022glm,falcon} necessitate function approximation, which is also important to our analysis itself.\footnote{To inspire future works on the analysis of knowledge transfer with function approximation, we discuss a potential pathway and related issues in Appendix~\ref{sec:funcApprox}.} For example, the equivalent effectiveness of minimizing the $\mathsf{CE}$ or $\mathsf{SEL}$ loss (with \datathree{}) from an alphabet-matching perspective may be incomplete after incorporating the \emph{approximation error}, which may consequently corroborate the empirical superior of the vanilla $\mathsf{SEL}$ loss \citep{ba2014deep}.
    \item All our upper bounds do not adapt to the student initialization strategy for unseen inputs or labels; while the student may be pre-trained in practice \citep{gu2023knowledge, jiang2023lion}. Thus it is vital to incorporate the skillfulness of the student before transfer to the convergence analysis.
\end{itemize}

\clearpage

\bibliographystyle{ims}
\bibliography{reference}

\begin{thebibliography}{105}
\expandafter\ifx\csname natexlab\endcsname\relax\def\natexlab#1{#1}\fi
\expandafter\ifx\csname url\endcsname\relax
  \def\url#1{\texttt{#1}}\fi
\expandafter\ifx\csname urlprefix\endcsname\relax\def\urlprefix{URL }\fi

\bibitem[{Abbeel and Ng(2004)}]{abbeel2004apprenticeship}
\textsc{Abbeel, P.} and \textsc{Ng, A.~Y.} (2004).
\newblock Apprenticeship learning via inverse reinforcement learning.
\newblock In \textit{Proceedings of the twenty-first international conference on Machine learning}.

\bibitem[{Agarwal et~al.(2023)Agarwal, Vieillard, Stanczyk, Ramos, Geist and Bachem}]{agarwal2023gkd}
\textsc{Agarwal, R.}, \textsc{Vieillard, N.}, \textsc{Stanczyk, P.}, \textsc{Ramos, S.}, \textsc{Geist, M.} and \textsc{Bachem, O.} (2023).
\newblock Gkd: Generalized knowledge distillation for auto-regressive sequence models.
\newblock \textit{arXiv preprint arXiv:2306.13649} .

\bibitem[{Allen-Zhu and Li(2020)}]{allen2020towards}
\textsc{Allen-Zhu, Z.} and \textsc{Li, Y.} (2020).
\newblock Towards understanding ensemble, knowledge distillation and self-distillation in deep learning.
\newblock \textit{arXiv preprint arXiv:2012.09816} .

\bibitem[{Almazrouei et~al.(2023)Almazrouei, Alobeidli, Alshamsi, Cappelli, Cojocaru, Alhammadi, Daniele, Heslow, Launay, Malartic, Noune, Pannier and Penedo}]{falcon}
\textsc{Almazrouei, E.}, \textsc{Alobeidli, H.}, \textsc{Alshamsi, A.}, \textsc{Cappelli, A.}, \textsc{Cojocaru, R.}, \textsc{Alhammadi, M.}, \textsc{Daniele, M.}, \textsc{Heslow, D.}, \textsc{Launay, J.}, \textsc{Malartic, Q.}, \textsc{Noune, B.}, \textsc{Pannier, B.} and \textsc{Penedo, G.} (2023).
\newblock The falcon series of language models: Towards open frontier models .

\bibitem[{Alyafeai et~al.(2020)Alyafeai, AlShaibani and Ahmad}]{alyafeai2020survey}
\textsc{Alyafeai, Z.}, \textsc{AlShaibani, M.~S.} and \textsc{Ahmad, I.} (2020).
\newblock A survey on transfer learning in natural language processing.
\newblock \textit{arXiv preprint arXiv:2007.04239} .

\bibitem[{Anthropic(2023)}]{anthropic2023claude2}
\textsc{Anthropic} (2023).
\newblock Model card and evaluations for claude models.
\newblock Accessed: Sep. 27,2023.

\bibitem[{Ba and Caruana(2014)}]{ba2014deep}
\textsc{Ba, J.} and \textsc{Caruana, R.} (2014).
\newblock Do deep nets really need to be deep?
\newblock \textit{Advances in neural information processing systems} \textbf{27}.

\bibitem[{Bommasani et~al.(2021)Bommasani, Hudson, Adeli, Altman, Arora, von Arx, Bernstein, Bohg, Bosselut, Brunskill et~al.}]{bommasani2021opportunities}
\textsc{Bommasani, R.}, \textsc{Hudson, D.~A.}, \textsc{Adeli, E.}, \textsc{Altman, R.}, \textsc{Arora, S.}, \textsc{von Arx, S.}, \textsc{Bernstein, M.~S.}, \textsc{Bohg, J.}, \textsc{Bosselut, A.}, \textsc{Brunskill, E.} \textsc{et~al.} (2021).
\newblock On the opportunities and risks of foundation models.
\newblock \textit{arXiv preprint arXiv:2108.07258} .

\bibitem[{Breiman and Shang(1996)}]{breiman1996born}
\textsc{Breiman, L.} and \textsc{Shang, N.} (1996).
\newblock Born again trees.
\newblock \textit{University of California, Berkeley, Berkeley, CA, Technical Report} \textbf{1} 4.

\bibitem[{Bretagnolle and Huber(1979)}]{bretagnolle1979estimation}
\textsc{Bretagnolle, J.} and \textsc{Huber, C.} (1979).
\newblock Estimation des densit{\'e}s: risque minimax.
\newblock \textit{Zeitschrift f{\"u}r Wahrscheinlichkeitstheorie und verwandte Gebiete} \textbf{47} 119--137.

\bibitem[{Brown et~al.(2019)Brown, Goo, Nagarajan and Niekum}]{brown2019extrapolating}
\textsc{Brown, D.}, \textsc{Goo, W.}, \textsc{Nagarajan, P.} and \textsc{Niekum, S.} (2019).
\newblock Extrapolating beyond suboptimal demonstrations via inverse reinforcement learning from observations.
\newblock In \textit{International conference on machine learning}. PMLR.

\bibitem[{Brown et~al.(2020)Brown, Mann, Ryder, Subbiah, Kaplan, Dhariwal, Neelakantan, Shyam, Sastry, Askell et~al.}]{brown2020language}
\textsc{Brown, T.}, \textsc{Mann, B.}, \textsc{Ryder, N.}, \textsc{Subbiah, M.}, \textsc{Kaplan, J.~D.}, \textsc{Dhariwal, P.}, \textsc{Neelakantan, A.}, \textsc{Shyam, P.}, \textsc{Sastry, G.}, \textsc{Askell, A.} \textsc{et~al.} (2020).
\newblock Language models are few-shot learners.
\newblock \textit{Advances in neural information processing systems} \textbf{33} 1877--1901.

\bibitem[{Bu et~al.(2020)Bu, Gao, Zou and Veeravalli}]{bu2020information}
\textsc{Bu, Y.}, \textsc{Gao, W.}, \textsc{Zou, S.} and \textsc{Veeravalli, V.} (2020).
\newblock Information-theoretic understanding of population risk improvement with model compression.
\newblock In \textit{Proceedings of the AAAI Conference on Artificial Intelligence}, vol.~34.

\bibitem[{Buciluǎ et~al.(2006)Buciluǎ, Caruana and Niculescu-Mizil}]{buciluǎ2006model}
\textsc{Buciluǎ, C.}, \textsc{Caruana, R.} and \textsc{Niculescu-Mizil, A.} (2006).
\newblock Model compression.
\newblock In \textit{Proceedings of the 12th ACM SIGKDD international conference on Knowledge discovery and data mining}.

\bibitem[{Burges and Sch{\"o}lkopf(1996)}]{burges1996improving}
\textsc{Burges, C.~J.} and \textsc{Sch{\"o}lkopf, B.} (1996).
\newblock Improving the accuracy and speed of support vector machines.
\newblock \textit{Advances in neural information processing systems} \textbf{9}.

\bibitem[{Canonne(2020)}]{canonne2020short}
\textsc{Canonne, C.~L.} (2020).
\newblock A short note on learning discrete distributions.
\newblock \textit{arXiv preprint arXiv:2002.11457} .

\bibitem[{Cesa-Bianchi et~al.(1997)Cesa-Bianchi, Freund, Haussler, Helmbold, Schapire and Warmuth}]{cesa1997use}
\textsc{Cesa-Bianchi, N.}, \textsc{Freund, Y.}, \textsc{Haussler, D.}, \textsc{Helmbold, D.~P.}, \textsc{Schapire, R.~E.} and \textsc{Warmuth, M.~K.} (1997).
\newblock How to use expert advice.
\newblock \textit{Journal of the ACM (JACM)} \textbf{44} 427--485.

\bibitem[{Chen et~al.(2021)Chen, Paleja and Gombolay}]{chen2021learning}
\textsc{Chen, L.}, \textsc{Paleja, R.} and \textsc{Gombolay, M.} (2021).
\newblock Learning from suboptimal demonstration via self-supervised reward regression.
\newblock In \textit{Conference on robot learning}. PMLR.

\bibitem[{Cheng et~al.(2020)Cheng, Rao, Chen and Zhang}]{cheng2020explaining}
\textsc{Cheng, X.}, \textsc{Rao, Z.}, \textsc{Chen, Y.} and \textsc{Zhang, Q.} (2020).
\newblock Explaining knowledge distillation by quantifying the knowledge.
\newblock In \textit{Proceedings of the IEEE/CVF conference on computer vision and pattern recognition}.

\bibitem[{Cho and Hariharan(2019)}]{cho2019efficacy}
\textsc{Cho, J.~H.} and \textsc{Hariharan, B.} (2019).
\newblock On the efficacy of knowledge distillation.
\newblock In \textit{Proceedings of the IEEE/CVF international conference on computer vision}.

\bibitem[{Dao et~al.(2021)Dao, Kamath, Syrgkanis and Mackey}]{dao2021knowledge}
\textsc{Dao, T.}, \textsc{Kamath, G.~M.}, \textsc{Syrgkanis, V.} and \textsc{Mackey, L.} (2021).
\newblock Knowledge distillation as semiparametric inference.
\newblock \textit{arXiv preprint arXiv:2104.09732} .

\bibitem[{Durrett(2019)}]{durrett2019probability}
\textsc{Durrett, R.} (2019).
\newblock \textit{Probability: theory and examples}, vol.~49.
\newblock Cambridge university press.

\bibitem[{Fang et~al.(2022)Fang, Mo, Wang, Song, Bei, Zhang and Song}]{fang2022up}
\textsc{Fang, G.}, \textsc{Mo, K.}, \textsc{Wang, X.}, \textsc{Song, J.}, \textsc{Bei, S.}, \textsc{Zhang, H.} and \textsc{Song, M.} (2022).
\newblock Up to 100x faster data-free knowledge distillation.
\newblock In \textit{Proceedings of the AAAI Conference on Artificial Intelligence}, vol.~36.

\bibitem[{Freund and Schapire(1995)}]{freund1995desicion}
\textsc{Freund, Y.} and \textsc{Schapire, R.~E.} (1995).
\newblock A desicion-theoretic generalization of on-line learning and an application to boosting.
\newblock In \textit{European conference on computational learning theory}. Springer.

\bibitem[{Furlanello et~al.(2018)Furlanello, Lipton, Tschannen, Itti and Anandkumar}]{furlanello2018born}
\textsc{Furlanello, T.}, \textsc{Lipton, Z.}, \textsc{Tschannen, M.}, \textsc{Itti, L.} and \textsc{Anandkumar, A.} (2018).
\newblock Born again neural networks.
\newblock In \textit{International Conference on Machine Learning}. PMLR.

\bibitem[{Gao et~al.(2023)Gao, Schulman and Hilton}]{gao2023scaling}
\textsc{Gao, L.}, \textsc{Schulman, J.} and \textsc{Hilton, J.} (2023).
\newblock Scaling laws for reward model overoptimization.
\newblock In \textit{International Conference on Machine Learning}. PMLR.

\bibitem[{Gao et~al.(2021)Gao, Tow, Biderman, Black, DiPofi, Foster, Golding, Hsu, McDonell, Muennighoff, Phang, Reynolds, Tang, Thite, Wang, Wang and Zou}]{eval-harness}
\textsc{Gao, L.}, \textsc{Tow, J.}, \textsc{Biderman, S.}, \textsc{Black, S.}, \textsc{DiPofi, A.}, \textsc{Foster, C.}, \textsc{Golding, L.}, \textsc{Hsu, J.}, \textsc{McDonell, K.}, \textsc{Muennighoff, N.}, \textsc{Phang, J.}, \textsc{Reynolds, L.}, \textsc{Tang, E.}, \textsc{Thite, A.}, \textsc{Wang, B.}, \textsc{Wang, K.} and \textsc{Zou, A.} (2021).
\newblock A framework for few-shot language model evaluation.

\bibitem[{Gou et~al.(2021)Gou, Yu, Maybank and Tao}]{gou2021knowledge}
\textsc{Gou, J.}, \textsc{Yu, B.}, \textsc{Maybank, S.~J.} and \textsc{Tao, D.} (2021).
\newblock Knowledge distillation: A survey.
\newblock \textit{International Journal of Computer Vision} \textbf{129} 1789--1819.

\bibitem[{Gu et~al.(2023{\natexlab{a}})Gu, Zhai, Zhang, Liu and Susskind}]{gu2023boot}
\textsc{Gu, J.}, \textsc{Zhai, S.}, \textsc{Zhang, Y.}, \textsc{Liu, L.} and \textsc{Susskind, J.~M.} (2023{\natexlab{a}}).
\newblock Boot: Data-free distillation of denoising diffusion models with bootstrapping.
\newblock In \textit{ICML 2023 Workshop on Structured Probabilistic Inference $\{$$\backslash$\&$\}$ Generative Modeling}.

\bibitem[{Gu et~al.(2023{\natexlab{b}})Gu, Dong, Wei and Huang}]{gu2023knowledge}
\textsc{Gu, Y.}, \textsc{Dong, L.}, \textsc{Wei, F.} and \textsc{Huang, M.} (2023{\natexlab{b}}).
\newblock Knowledge distillation of large language models.
\newblock \textit{arXiv preprint arXiv:2306.08543} .

\bibitem[{Han et~al.(2023)Han, Kim, Choi and Yoon}]{han2023impact}
\textsc{Han, H.}, \textsc{Kim, S.}, \textsc{Choi, H.-S.} and \textsc{Yoon, S.} (2023).
\newblock On the impact of knowledge distillation for model interpretability.
\newblock \textit{arXiv preprint arXiv:2305.15734} .

\bibitem[{Harutyunyan et~al.(2023)Harutyunyan, Rawat, Menon, Kim and Kumar}]{harutyunyan2023supervision}
\textsc{Harutyunyan, H.}, \textsc{Rawat, A.~S.}, \textsc{Menon, A.~K.}, \textsc{Kim, S.} and \textsc{Kumar, S.} (2023).
\newblock Supervision complexity and its role in knowledge distillation.
\newblock \textit{arXiv preprint arXiv:2301.12245} .

\bibitem[{Hinton et~al.(2015)Hinton, Vinyals and Dean}]{hinton2015distilling}
\textsc{Hinton, G.}, \textsc{Vinyals, O.} and \textsc{Dean, J.} (2015).
\newblock Distilling the knowledge in a neural network.
\newblock \textit{arXiv preprint arXiv:1503.02531} .

\bibitem[{Hsu et~al.(2021)Hsu, Ji, Telgarsky and Wang}]{hsu2021generalization}
\textsc{Hsu, D.}, \textsc{Ji, Z.}, \textsc{Telgarsky, M.} and \textsc{Wang, L.} (2021).
\newblock Generalization bounds via distillation.
\newblock \textit{arXiv preprint arXiv:2104.05641} .

\bibitem[{Huang and Wang(2017)}]{huang2017like}
\textsc{Huang, Z.} and \textsc{Wang, N.} (2017).
\newblock Like what you like: Knowledge distill via neuron selectivity transfer.
\newblock \textit{arXiv preprint arXiv:1707.01219} .

\bibitem[{Ji and Zhu(2020)}]{ji2020knowledge}
\textsc{Ji, G.} and \textsc{Zhu, Z.} (2020).
\newblock Knowledge distillation in wide neural networks: Risk bound, data efficiency and imperfect teacher.
\newblock \textit{Advances in Neural Information Processing Systems} \textbf{33} 20823--20833.

\bibitem[{Jiang et~al.(2019)Jiang, He, Chen, Liu, Gao and Zhao}]{jiang2019smart}
\textsc{Jiang, H.}, \textsc{He, P.}, \textsc{Chen, W.}, \textsc{Liu, X.}, \textsc{Gao, J.} and \textsc{Zhao, T.} (2019).
\newblock Smart: Robust and efficient fine-tuning for pre-trained natural language models through principled regularized optimization.
\newblock \textit{arXiv preprint arXiv:1911.03437} .

\bibitem[{Jiang et~al.(2023)Jiang, Chan, Chen and Wang}]{jiang2023lion}
\textsc{Jiang, Y.}, \textsc{Chan, C.}, \textsc{Chen, M.} and \textsc{Wang, W.} (2023).
\newblock Lion: Adversarial distillation of closed-source large language model.
\newblock \textit{arXiv preprint arXiv:2305.12870} .

\bibitem[{Kim et~al.(2021)Kim, Oh, Kim, Cho and Yun}]{kim2021comparing}
\textsc{Kim, T.}, \textsc{Oh, J.}, \textsc{Kim, N.}, \textsc{Cho, S.} and \textsc{Yun, S.-Y.} (2021).
\newblock Comparing kullback-leibler divergence and mean squared error loss in knowledge distillation.
\newblock \textit{arXiv preprint arXiv:2105.08919} .

\bibitem[{Li et~al.(2014)Li, Zhao, Huang and Gong}]{li2014learning}
\textsc{Li, J.}, \textsc{Zhao, R.}, \textsc{Huang, J.-T.} and \textsc{Gong, Y.} (2014).
\newblock Learning small-size dnn with output-distribution-based criteria.
\newblock In \textit{Fifteenth annual conference of the international speech communication association}.

\bibitem[{Li et~al.(2023)Li, Zhang, Dubois, Taori, Gulrajani, Guestrin, Liang and Hashimoto}]{alpaca_eval}
\textsc{Li, X.}, \textsc{Zhang, T.}, \textsc{Dubois, Y.}, \textsc{Taori, R.}, \textsc{Gulrajani, I.}, \textsc{Guestrin, C.}, \textsc{Liang, P.} and \textsc{Hashimoto, T.~B.} (2023).
\newblock Alpacaeval: An automatic evaluator of instruction-following models.
\newblock \url{https://github.com/tatsu-lab/alpaca_eval}.

\bibitem[{Liang et~al.(2023)Liang, Zuo, Zhang, He, Chen and Zhao}]{liang2023less}
\textsc{Liang, C.}, \textsc{Zuo, S.}, \textsc{Zhang, Q.}, \textsc{He, P.}, \textsc{Chen, W.} and \textsc{Zhao, T.} (2023).
\newblock Less is more: Task-aware layer-wise distillation for language model compression.
\newblock In \textit{International Conference on Machine Learning}. PMLR.

\bibitem[{Liu et~al.(2023)Liu, Li, Wu and Lee}]{liu2023visual}
\textsc{Liu, H.}, \textsc{Li, C.}, \textsc{Wu, Q.} and \textsc{Lee, Y.~J.} (2023).
\newblock Visual instruction tuning.
\newblock \textit{arXiv preprint arXiv:2304.08485} .

\bibitem[{Lopes et~al.(2017)Lopes, Fenu and Starner}]{lopes2017data}
\textsc{Lopes, R.~G.}, \textsc{Fenu, S.} and \textsc{Starner, T.} (2017).
\newblock Data-free knowledge distillation for deep neural networks.
\newblock \textit{arXiv preprint arXiv:1710.07535} .

\bibitem[{Lopez-Paz et~al.(2015)Lopez-Paz, Bottou, Sch{\"o}lkopf and Vapnik}]{lopez2015unifying}
\textsc{Lopez-Paz, D.}, \textsc{Bottou, L.}, \textsc{Sch{\"o}lkopf, B.} and \textsc{Vapnik, V.} (2015).
\newblock Unifying distillation and privileged information.
\newblock \textit{arXiv preprint arXiv:1511.03643} .

\bibitem[{McAllester and Ortiz(2003)}]{mcallester2003concentration}
\textsc{McAllester, D.} and \textsc{Ortiz, L.} (2003).
\newblock Concentration inequalities for the missing mass and for histogram rule error.
\newblock \textit{Journal of Machine Learning Research} \textbf{4} 895--911.

\bibitem[{Menon et~al.(2020)Menon, Rawat, Reddi, Kim and Kumar}]{menon2020distillation}
\textsc{Menon, A.~K.}, \textsc{Rawat, A.~S.}, \textsc{Reddi, S.~J.}, \textsc{Kim, S.} and \textsc{Kumar, S.} (2020).
\newblock Why distillation helps: a statistical perspective.
\newblock \textit{arXiv preprint arXiv:2005.10419} .

\bibitem[{Mitzenmacher and Upfal(2017)}]{mitzenmacher2017probability}
\textsc{Mitzenmacher, M.} and \textsc{Upfal, E.} (2017).
\newblock \textit{Probability and computing: Randomization and probabilistic techniques in algorithms and data analysis}.
\newblock Cambridge university press.

\bibitem[{Mobahi et~al.(2020)Mobahi, Farajtabar and Bartlett}]{mobahi2020self}
\textsc{Mobahi, H.}, \textsc{Farajtabar, M.} and \textsc{Bartlett, P.} (2020).
\newblock Self-distillation amplifies regularization in hilbert space.
\newblock \textit{Advances in Neural Information Processing Systems} \textbf{33} 3351--3361.

\bibitem[{M{\"u}ller et~al.(2020)M{\"u}ller, Kornblith and Hinton}]{muller2020subclass}
\textsc{M{\"u}ller, R.}, \textsc{Kornblith, S.} and \textsc{Hinton, G.} (2020).
\newblock Subclass distillation.
\newblock \textit{arXiv preprint arXiv:2002.03936} .

\bibitem[{Nayak et~al.(2019)Nayak, Mopuri, Shaj, Radhakrishnan and Chakraborty}]{nayak2019zero}
\textsc{Nayak, G.~K.}, \textsc{Mopuri, K.~R.}, \textsc{Shaj, V.}, \textsc{Radhakrishnan, V.~B.} and \textsc{Chakraborty, A.} (2019).
\newblock Zero-shot knowledge distillation in deep networks.
\newblock In \textit{International Conference on Machine Learning}. PMLR.

\bibitem[{Ng et~al.(2000)Ng, Russell et~al.}]{ng2000algorithms}
\textsc{Ng, A.~Y.}, \textsc{Russell, S.} \textsc{et~al.} (2000).
\newblock Algorithms for inverse reinforcement learning.
\newblock In \textit{Icml}, vol.~1.

\bibitem[{Nguyen et~al.(2022)Nguyen, Gupta, Do and Venkatesh}]{nguyen2022black}
\textsc{Nguyen, D.}, \textsc{Gupta, S.}, \textsc{Do, K.} and \textsc{Venkatesh, S.} (2022).
\newblock Black-box few-shot knowledge distillation.
\newblock In \textit{European Conference on Computer Vision}. Springer.

\bibitem[{OpenAI(2023{\natexlab{a}})}]{openai2023dall-e-3}
\textsc{OpenAI} (2023{\natexlab{a}}).
\newblock Accessed: Sep. 28,2023.

\bibitem[{OpenAI(2023{\natexlab{b}})}]{openai2023gpt4}
\textsc{OpenAI} (2023{\natexlab{b}}).
\newblock Gpt-4 technical report.

\bibitem[{Orekondy et~al.(2019)Orekondy, Schiele and Fritz}]{orekondy2019knockoff}
\textsc{Orekondy, T.}, \textsc{Schiele, B.} and \textsc{Fritz, M.} (2019).
\newblock Knockoff nets: Stealing functionality of black-box models.
\newblock In \textit{Proceedings of the IEEE/CVF conference on computer vision and pattern recognition}.

\bibitem[{Ouyang et~al.(2022)Ouyang, Wu, Jiang, Almeida, Wainwright, Mishkin, Zhang, Agarwal, Slama, Ray et~al.}]{ouyang2022training}
\textsc{Ouyang, L.}, \textsc{Wu, J.}, \textsc{Jiang, X.}, \textsc{Almeida, D.}, \textsc{Wainwright, C.}, \textsc{Mishkin, P.}, \textsc{Zhang, C.}, \textsc{Agarwal, S.}, \textsc{Slama, K.}, \textsc{Ray, A.} \textsc{et~al.} (2022).
\newblock Training language models to follow instructions with human feedback.
\newblock \textit{Advances in Neural Information Processing Systems} \textbf{35} 27730--27744.

\bibitem[{Panahi et~al.(2022)Panahi, Rahbar, Bhattacharyya, Dubhashi and Haghir~Chehreghani}]{panahi2022analysis}
\textsc{Panahi, A.}, \textsc{Rahbar, A.}, \textsc{Bhattacharyya, C.}, \textsc{Dubhashi, D.} and \textsc{Haghir~Chehreghani, M.} (2022).
\newblock Analysis of knowledge transfer in kernel regime.
\newblock In \textit{Proceedings of the 31st ACM International Conference on Information \& Knowledge Management}.

\bibitem[{Paninski(2008)}]{paninski2008coincidence}
\textsc{Paninski, L.} (2008).
\newblock A coincidence-based test for uniformity given very sparsely sampled discrete data.
\newblock \textit{IEEE Transactions on Information Theory} \textbf{54} 4750--4755.

\bibitem[{Park et~al.(2019)Park, Kim, Lu and Cho}]{park2019relational}
\textsc{Park, W.}, \textsc{Kim, D.}, \textsc{Lu, Y.} and \textsc{Cho, M.} (2019).
\newblock Relational knowledge distillation.
\newblock In \textit{Proceedings of the IEEE/CVF conference on computer vision and pattern recognition}.

\bibitem[{Paszke et~al.(2019)Paszke, Gross, Massa, Lerer, Bradbury, Chanan, Killeen, Lin, Gimelshein, Antiga, Desmaison, Kopf, Yang, DeVito, Raison, Tejani, Chilamkurthy, Steiner, Fang, Bai and Chintala}]{Paszke_PyTorch_An_Imperative_2019}
\textsc{Paszke, A.}, \textsc{Gross, S.}, \textsc{Massa, F.}, \textsc{Lerer, A.}, \textsc{Bradbury, J.}, \textsc{Chanan, G.}, \textsc{Killeen, T.}, \textsc{Lin, Z.}, \textsc{Gimelshein, N.}, \textsc{Antiga, L.}, \textsc{Desmaison, A.}, \textsc{Kopf, A.}, \textsc{Yang, E.}, \textsc{DeVito, Z.}, \textsc{Raison, M.}, \textsc{Tejani, A.}, \textsc{Chilamkurthy, S.}, \textsc{Steiner, B.}, \textsc{Fang, L.}, \textsc{Bai, J.} and \textsc{Chintala, S.} (2019).
\newblock {PyTorch: An Imperative Style, High-Performance Deep Learning Library}.
\newblock In \textit{Advances in Neural Information Processing Systems 32} (H.~Wallach, H.~Larochelle, A.~Beygelzimer, F.~d'Alché Buc, E.~Fox and R.~Garnett, eds.). Curran Associates, Inc.

\bibitem[{Peng et~al.(2023)Peng, Li, He, Galley and Gao}]{peng2023instruction}
\textsc{Peng, B.}, \textsc{Li, C.}, \textsc{He, P.}, \textsc{Galley, M.} and \textsc{Gao, J.} (2023).
\newblock Instruction tuning with gpt-4.
\newblock \textit{arXiv preprint arXiv:2304.03277} .

\bibitem[{Phuong and Lampert(2019)}]{phuong2019towards}
\textsc{Phuong, M.} and \textsc{Lampert, C.} (2019).
\newblock Towards understanding knowledge distillation.
\newblock In \textit{International conference on machine learning}. PMLR.

\bibitem[{Polyanskiy and Wu(2022)}]{polyanskiy2022information}
\textsc{Polyanskiy, Y.} and \textsc{Wu, Y.} (2022).
\newblock Information theory: From coding to learning.
\newblock \textit{Book draft} .

\bibitem[{Qiu et~al.(2022)Qiu, Ma, Yang, Liu, Hou, Yi and Ouyang}]{qiu2022better}
\textsc{Qiu, Z.}, \textsc{Ma, X.}, \textsc{Yang, K.}, \textsc{Liu, C.}, \textsc{Hou, J.}, \textsc{Yi, S.} and \textsc{Ouyang, W.} (2022).
\newblock Better teacher better student: Dynamic prior knowledge for knowledge distillation.
\newblock \textit{arXiv preprint arXiv:2206.06067} .

\bibitem[{Radford et~al.(2019)Radford, Wu, Child, Luan, Amodei, Sutskever et~al.}]{radford2019language}
\textsc{Radford, A.}, \textsc{Wu, J.}, \textsc{Child, R.}, \textsc{Luan, D.}, \textsc{Amodei, D.}, \textsc{Sutskever, I.} \textsc{et~al.} (2019).
\newblock Language models are unsupervised multitask learners.
\newblock \textit{OpenAI blog} \textbf{1} 9.

\bibitem[{Rajaraman et~al.(2020)Rajaraman, Yang, Jiao and Ramachandran}]{rajaraman2020toward}
\textsc{Rajaraman, N.}, \textsc{Yang, L.~F.}, \textsc{Jiao, J.} and \textsc{Ramachandran, K.} (2020).
\newblock Toward the fundamental limits of imitation learning.
\newblock \textit{arXiv preprint arXiv:2009.05990} .

\bibitem[{Ramesh et~al.(2022)Ramesh, Dhariwal, Nichol, Chu and Chen}]{ramesh2022hierarchical}
\textsc{Ramesh, A.}, \textsc{Dhariwal, P.}, \textsc{Nichol, A.}, \textsc{Chu, C.} and \textsc{Chen, M.} (2022).
\newblock Hierarchical text-conditional image generation with clip latents.
\newblock \textit{arXiv preprint arXiv:2204.06125} \textbf{1} 3.

\bibitem[{Ramesh et~al.(2021)Ramesh, Pavlov, Goh, Gray, Voss, Radford, Chen and Sutskever}]{ramesh2021zero}
\textsc{Ramesh, A.}, \textsc{Pavlov, M.}, \textsc{Goh, G.}, \textsc{Gray, S.}, \textsc{Voss, C.}, \textsc{Radford, A.}, \textsc{Chen, M.} and \textsc{Sutskever, I.} (2021).
\newblock Zero-shot text-to-image generation.
\newblock In \textit{International Conference on Machine Learning}. PMLR.

\bibitem[{Rashidinejad et~al.(2021)Rashidinejad, Zhu, Ma, Jiao and Russell}]{rashidinejad2021bridging}
\textsc{Rashidinejad, P.}, \textsc{Zhu, B.}, \textsc{Ma, C.}, \textsc{Jiao, J.} and \textsc{Russell, S.} (2021).
\newblock Bridging offline reinforcement learning and imitation learning: A tale of pessimism.
\newblock \textit{Advances in Neural Information Processing Systems} \textbf{34} 11702--11716.

\bibitem[{Romero et~al.(2014)Romero, Ballas, Kahou, Chassang, Gatta and Bengio}]{romero2014fitnets}
\textsc{Romero, A.}, \textsc{Ballas, N.}, \textsc{Kahou, S.~E.}, \textsc{Chassang, A.}, \textsc{Gatta, C.} and \textsc{Bengio, Y.} (2014).
\newblock Fitnets: Hints for thin deep nets.
\newblock \textit{arXiv preprint arXiv:1412.6550} .

\bibitem[{Ross and Bagnell(2010)}]{ross2010efficient}
\textsc{Ross, S.} and \textsc{Bagnell, D.} (2010).
\newblock Efficient reductions for imitation learning.
\newblock In \textit{Proceedings of the thirteenth international conference on artificial intelligence and statistics}. JMLR Workshop and Conference Proceedings.

\bibitem[{Shalev-Shwartz and Ben-David(2014)}]{shalev2014understanding}
\textsc{Shalev-Shwartz, S.} and \textsc{Ben-David, S.} (2014).
\newblock \textit{Understanding machine learning: From theory to algorithms}.
\newblock Cambridge university press.

\bibitem[{Srinivas and Fleuret(2018)}]{srinivas2018knowledge}
\textsc{Srinivas, S.} and \textsc{Fleuret, F.} (2018).
\newblock Knowledge transfer with jacobian matching.
\newblock In \textit{International Conference on Machine Learning}. PMLR.

\bibitem[{Tang et~al.(2020)Tang, Shivanna, Zhao, Lin, Singh, Chi and Jain}]{tang2020understanding}
\textsc{Tang, J.}, \textsc{Shivanna, R.}, \textsc{Zhao, Z.}, \textsc{Lin, D.}, \textsc{Singh, A.}, \textsc{Chi, E.~H.} and \textsc{Jain, S.} (2020).
\newblock Understanding and improving knowledge distillation.
\newblock \textit{arXiv preprint arXiv:2002.03532} .

\bibitem[{Tian et~al.(2019)Tian, Krishnan and Isola}]{tian2019contrastive}
\textsc{Tian, Y.}, \textsc{Krishnan, D.} and \textsc{Isola, P.} (2019).
\newblock Contrastive representation distillation.
\newblock \textit{arXiv preprint arXiv:1910.10699} .

\bibitem[{Tian et~al.(2023)Tian, Pei, Zhang, Zhang and Chawla}]{tian2023knowledge}
\textsc{Tian, Y.}, \textsc{Pei, S.}, \textsc{Zhang, X.}, \textsc{Zhang, C.} and \textsc{Chawla, N.~V.} (2023).
\newblock Knowledge distillation on graphs: A survey.
\newblock \textit{arXiv preprint arXiv:2302.00219} .

\bibitem[{Touvron et~al.(2023{\natexlab{a}})Touvron, Lavril, Izacard, Martinet, Lachaux, Lacroix, Rozi{\`e}re, Goyal, Hambro, Azhar et~al.}]{touvron2023llama}
\textsc{Touvron, H.}, \textsc{Lavril, T.}, \textsc{Izacard, G.}, \textsc{Martinet, X.}, \textsc{Lachaux, M.-A.}, \textsc{Lacroix, T.}, \textsc{Rozi{\`e}re, B.}, \textsc{Goyal, N.}, \textsc{Hambro, E.}, \textsc{Azhar, F.} \textsc{et~al.} (2023{\natexlab{a}}).
\newblock Llama: Open and efficient foundation language models.
\newblock \textit{arXiv preprint arXiv:2302.13971} .

\bibitem[{Touvron et~al.(2023{\natexlab{b}})Touvron, Martin, Stone, Albert, Almahairi, Babaei, Bashlykov, Batra, Bhargava, Bhosale et~al.}]{touvron2023llama2}
\textsc{Touvron, H.}, \textsc{Martin, L.}, \textsc{Stone, K.}, \textsc{Albert, P.}, \textsc{Almahairi, A.}, \textsc{Babaei, Y.}, \textsc{Bashlykov, N.}, \textsc{Batra, S.}, \textsc{Bhargava, P.}, \textsc{Bhosale, S.} \textsc{et~al.} (2023{\natexlab{b}}).
\newblock Llama 2: Open foundation and fine-tuned chat models.
\newblock \textit{arXiv preprint arXiv:2307.09288} .

\bibitem[{Tsybakov(2003)}]{tsybakov2003optimal}
\textsc{Tsybakov, A.~B.} (2003).
\newblock Optimal rates of aggregation.
\newblock In \textit{Learning Theory and Kernel Machines: 16th Annual Conference on Learning Theory and 7th Kernel Workshop, COLT/Kernel 2003, Washington, DC, USA, August 24-27, 2003. Proceedings}. Springer.

\bibitem[{Tsybakov(2009)}]{tsybakov2009nonpara}
\textsc{Tsybakov, A.~B.} (2009).
\newblock \textit{Introduction to Nonparametric Estimation}.
\newblock Springer series in statistics, Springer.

\bibitem[{Tung and Mori(2019)}]{tung2019similarity}
\textsc{Tung, F.} and \textsc{Mori, G.} (2019).
\newblock Similarity-preserving knowledge distillation.
\newblock In \textit{Proceedings of the IEEE/CVF international conference on computer vision}.

\bibitem[{Van~der Vaart(2000)}]{van2000asymptotic}
\textsc{Van~der Vaart, A.~W.} (2000).
\newblock \textit{Asymptotic statistics}, vol.~3.
\newblock Cambridge university press.

\bibitem[{van~der Vaart and Wellner(1996)}]{van1996empirical}
\textsc{van~der Vaart, A.~W.} and \textsc{Wellner, J.~A.} (1996).
\newblock Springer series in statistics.
\newblock \textit{Weak convergence and empirical processesSpringer, New York} .

\bibitem[{Vapnik et~al.(2015)Vapnik, Izmailov et~al.}]{vapnik2015learning}
\textsc{Vapnik, V.}, \textsc{Izmailov, R.} \textsc{et~al.} (2015).
\newblock Learning using privileged information: similarity control and knowledge transfer.
\newblock \textit{J. Mach. Learn. Res.} \textbf{16} 2023--2049.

\bibitem[{Wainwright(2019)}]{wainwright2019high}
\textsc{Wainwright, M.~J.} (2019).
\newblock \textit{High-dimensional statistics: A non-asymptotic viewpoint}, vol.~48.
\newblock Cambridge university press.

\bibitem[{Wang et~al.(2020)Wang, Li, Wang and Gong}]{wang2020neural}
\textsc{Wang, D.}, \textsc{Li, Y.}, \textsc{Wang, L.} and \textsc{Gong, B.} (2020).
\newblock Neural networks are more productive teachers than human raters: Active mixup for data-efficient knowledge distillation from a blackbox model.
\newblock In \textit{Proceedings of the IEEE/CVF Conference on Computer Vision and Pattern Recognition}.

\bibitem[{Wang et~al.(2023{\natexlab{a}})Wang, Cheng, Yu and Liu}]{openchat}
\textsc{Wang, G.}, \textsc{Cheng, S.}, \textsc{Yu, Q.} and \textsc{Liu, C.} (2023{\natexlab{a}}).
\newblock {OpenChat: Advancing Open-source Language Models with Imperfect Data}.

\bibitem[{Wang and Yoon(2021)}]{wang2021knowledge}
\textsc{Wang, L.} and \textsc{Yoon, K.-J.} (2021).
\newblock Knowledge distillation and student-teacher learning for visual intelligence: A review and new outlooks.
\newblock \textit{IEEE transactions on pattern analysis and machine intelligence} \textbf{44} 3048--3068.

\bibitem[{Wang et~al.(2023{\natexlab{b}})Wang, Ivison, Dasigi, Hessel, Khot, Chandu, Wadden, MacMillan, Smith, Beltagy and Hajishirzi}]{wang2023far}
\textsc{Wang, Y.}, \textsc{Ivison, H.}, \textsc{Dasigi, P.}, \textsc{Hessel, J.}, \textsc{Khot, T.}, \textsc{Chandu, K.~R.}, \textsc{Wadden, D.}, \textsc{MacMillan, K.}, \textsc{Smith, N.~A.}, \textsc{Beltagy, I.} and \textsc{Hajishirzi, H.} (2023{\natexlab{b}}).
\newblock How far can camels go? exploring the state of instruction tuning on open resources.

\bibitem[{Wang(2021)}]{wang2021zero}
\textsc{Wang, Z.} (2021).
\newblock Zero-shot knowledge distillation from a decision-based black-box model.
\newblock In \textit{International Conference on Machine Learning}. PMLR.

\bibitem[{Weiss et~al.(2016)Weiss, Khoshgoftaar and Wang}]{weiss2016survey}
\textsc{Weiss, K.}, \textsc{Khoshgoftaar, T.~M.} and \textsc{Wang, D.} (2016).
\newblock A survey of transfer learning.
\newblock \textit{Journal of Big data} \textbf{3} 1--40.

\bibitem[{Yim et~al.(2017)Yim, Joo, Bae and Kim}]{yim2017gift}
\textsc{Yim, J.}, \textsc{Joo, D.}, \textsc{Bae, J.} and \textsc{Kim, J.} (2017).
\newblock A gift from knowledge distillation: Fast optimization, network minimization and transfer learning.
\newblock In \textit{Proceedings of the IEEE conference on computer vision and pattern recognition}.

\bibitem[{Yin et~al.(2020)Yin, Molchanov, Alvarez, Li, Mallya, Hoiem, Jha and Kautz}]{yin2020dreaming}
\textsc{Yin, H.}, \textsc{Molchanov, P.}, \textsc{Alvarez, J.~M.}, \textsc{Li, Z.}, \textsc{Mallya, A.}, \textsc{Hoiem, D.}, \textsc{Jha, N.~K.} and \textsc{Kautz, J.} (2020).
\newblock Dreaming to distill: Data-free knowledge transfer via deepinversion.
\newblock In \textit{Proceedings of the IEEE/CVF Conference on Computer Vision and Pattern Recognition}.

\bibitem[{Yin et~al.(2021)Yin, Bai and Wang}]{yin2021near}
\textsc{Yin, M.}, \textsc{Bai, Y.} and \textsc{Wang, Y.-X.} (2021).
\newblock Near-optimal provable uniform convergence in offline policy evaluation for reinforcement learning.
\newblock In \textit{International Conference on Artificial Intelligence and Statistics}. PMLR.

\bibitem[{Yu(1997)}]{yu1997assouad}
\textsc{Yu, B.} (1997).
\newblock Assouad, fano, and le cam.
\newblock In \textit{Festschrift for Lucien Le Cam: research papers in probability and statistics}. Springer, 423--435.

\bibitem[{Yuan et~al.(2020)Yuan, Tay, Li, Wang and Feng}]{yuan2020revisiting}
\textsc{Yuan, L.}, \textsc{Tay, F.~E.}, \textsc{Li, G.}, \textsc{Wang, T.} and \textsc{Feng, J.} (2020).
\newblock Revisiting knowledge distillation via label smoothing regularization.
\newblock In \textit{Proceedings of the IEEE/CVF Conference on Computer Vision and Pattern Recognition}.

\bibitem[{Zagoruyko and Komodakis(2016)}]{zagoruyko2016paying}
\textsc{Zagoruyko, S.} and \textsc{Komodakis, N.} (2016).
\newblock Paying more attention to attention: Improving the performance of convolutional neural networks via attention transfer.
\newblock \textit{arXiv preprint arXiv:1612.03928} .

\bibitem[{Zeng et~al.(2022)Zeng, Liu, Du, Wang, Lai, Ding, Yang, Xu, Zheng, Xia et~al.}]{zeng2022glm}
\textsc{Zeng, A.}, \textsc{Liu, X.}, \textsc{Du, Z.}, \textsc{Wang, Z.}, \textsc{Lai, H.}, \textsc{Ding, M.}, \textsc{Yang, Z.}, \textsc{Xu, Y.}, \textsc{Zheng, W.}, \textsc{Xia, X.} \textsc{et~al.} (2022).
\newblock Glm-130b: An open bilingual pre-trained model.
\newblock \textit{arXiv preprint arXiv:2210.02414} .

\bibitem[{Zhao et~al.(2022)Zhao, Cui, Song, Qiu and Liang}]{zhao2022decoupled}
\textsc{Zhao, B.}, \textsc{Cui, Q.}, \textsc{Song, R.}, \textsc{Qiu, Y.} and \textsc{Liang, J.} (2022).
\newblock Decoupled knowledge distillation.
\newblock In \textit{Proceedings of the IEEE/CVF Conference on computer vision and pattern recognition}.

\bibitem[{Zheng et~al.(2023)Zheng, Chiang, Sheng, Zhuang, Wu, Zhuang, Lin, Li, Li, Xing, Zhang, Gonzalez and Stoica}]{zheng2023judging}
\textsc{Zheng, L.}, \textsc{Chiang, W.-L.}, \textsc{Sheng, Y.}, \textsc{Zhuang, S.}, \textsc{Wu, Z.}, \textsc{Zhuang, Y.}, \textsc{Lin, Z.}, \textsc{Li, Z.}, \textsc{Li, D.}, \textsc{Xing, E.~P.}, \textsc{Zhang, H.}, \textsc{Gonzalez, J.~E.} and \textsc{Stoica, I.} (2023).
\newblock Judging llm-as-a-judge with mt-bench and chatbot arena.

\bibitem[{Zhou et~al.(2021)Zhou, Song, Chen, Zhou, Wang, Yuan and Zhang}]{zhou2021rethinking}
\textsc{Zhou, H.}, \textsc{Song, L.}, \textsc{Chen, J.}, \textsc{Zhou, Y.}, \textsc{Wang, G.}, \textsc{Yuan, J.} and \textsc{Zhang, Q.} (2021).
\newblock Rethinking soft labels for knowledge distillation: A bias-variance tradeoff perspective.
\newblock \textit{arXiv preprint arXiv:2102.00650} .

\bibitem[{Zhu et~al.(2023{\natexlab{a}})Zhu, Jiao and Jordan}]{zhu2023principled}
\textsc{Zhu, B.}, \textsc{Jiao, J.} and \textsc{Jordan, M.~I.} (2023{\natexlab{a}}).
\newblock Principled reinforcement learning with human feedback from pairwise or $ k $-wise comparisons.
\newblock \textit{arXiv preprint arXiv:2301.11270} .

\bibitem[{Zhu et~al.(2023{\natexlab{b}})Zhu, Sharma, Frujeri, Dong, Zhu, Jordan and Jiao}]{zhu2023fine}
\textsc{Zhu, B.}, \textsc{Sharma, H.}, \textsc{Frujeri, F.~V.}, \textsc{Dong, S.}, \textsc{Zhu, C.}, \textsc{Jordan, M.~I.} and \textsc{Jiao, J.} (2023{\natexlab{b}}).
\newblock Fine-tuning language models with advantage-induced policy alignment.
\newblock \textit{arXiv preprint arXiv:2306.02231} .

\bibitem[{Zhu et~al.(2023{\natexlab{c}})Zhu, Lin, Jain and Zhou}]{zhu2023transfer}
\textsc{Zhu, Z.}, \textsc{Lin, K.}, \textsc{Jain, A.~K.} and \textsc{Zhou, J.} (2023{\natexlab{c}}).
\newblock Transfer learning in deep reinforcement learning: A survey.
\newblock \textit{IEEE Transactions on Pattern Analysis and Machine Intelligence} .

\end{thebibliography}

\clearpage

\appendix

\section{Additional Related Works}\label{sec:appendix:additionalRW}

We review key paradigms closely related to our formulation, especially KD; and point the reader to \citet{weiss2016survey, gou2021knowledge, alyafeai2020survey, zhu2023transfer,wang2021knowledge, tian2023knowledge} for numerous algorithms of knowledge transfer on modalities like sequences, images, graphs, etc. We restrict our survey to the single task of interest: classification.
\footnote{Topics related to knowledge transfer across different tasks or modalities are beyond the scope.}

\subsection{Inspiring Paradigms of Knowledge Transfer in Machine Learning}

\paragraph{Paradigms of Knowledge Transfer.}
Knowledge transfer is not a patent of neural nets with \texttt{softmax} as their last layers. Similar ideas have realizations for ensemble of classification trees \citep{breiman1996born} and even margin-based classifiers \citep{burges1996improving}. Since 2010s, there has been a line of work concentrating on 
the utilization of a trained teacher model and its original training set in a totally \emph{white-box} manner with the purpose of student accuracy improvement \citep{hinton2015distilling, furlanello2018born, cho2019efficacy, zhao2022decoupled, romero2014fitnets, yim2017gift, huang2017like, park2019relational, tian2019contrastive, tung2019similarity, qiu2022better}. As the computation budget becomes relatively tight with respect to the scale of datasets, another line of works propose to avoid using of the full dataset for teacher pre-training during KD, and resort to architecture-specific metadata \citep{lopes2017data}, synthetic data \citep{nayak2019zero, yin2020dreaming, fang2022up}, or bootstrapping \citep{gu2023boot} instead; which is dubbed the \emph{data-free} approach. The sagacious vision that the teacher architecture may be agnostic or the teacher ($\log$-)probability output may at least go through certain censorship does not receive enough attention from the community in the era of open sourcing \citep{orekondy2019knockoff, wang2020neural, wang2021zero, nguyen2022black}. However, after the debut of closed-source and game-changing foundation models, practitioners find it plausibly nice to train their own models to purely mimic the response of these strong teachers. For example, though OpenAI only exposes transparent APIs of ChatGPT \& GPT-4 \citep{openai2023gpt4} to customers, there has been a line of efforts towards distilling black-box language models without even accessing the last-layer logits \citep{zheng2023judging,wang2023far,openchat}. Some primary results \citep{openchat} show that only letting LLaMA \citep{touvron2023llama} mimic about 6000 carefully chosen trajectories generated by human-GPT-4 interactions can drastically boost the performance of these open-source autoregressive language models on common evaluation benchmarks \citep{eval-harness,alpaca_eval}.

\subsection{Our Formulation versus Previous Paradigms}\label{subsec:connections}

Motivated by the progress trend of knowledge transfer, we decouple the formulation of the input distribution $\rho$ from teacher training details and view the reference poicy $\pi^\star$ as the gold standard (ground truth), beyond just a proxy of it. In the following two paragraphs, we would like to also emphasize the distinctions between our formulations and 1) imitation learning (whose performance is measured by reward sub-optimality), or 2) the LUPI framework \citep{vapnik2015learning, lopez2015unifying}, respectively.

\emph{Optimality} is not explicitly defined for $\pi^\star$ in our formulation, yet we solely want to mimic (the conditional density of) the teacher, so $\pi^\star$ is dubbed the \emph{reference policy}. This view aligns with a recent belief that foundation models are by definition ``good'' and in effect black-box teachers, judges, and raters of the student ones \citep{peng2023instruction,liu2023visual, zheng2023judging}. Generally speaking, optimal reward maximization is neither sufficient nor necessary for efficient knowledge transfer ($\Leftrightarrow$ accurate behavioral imitation), because an optimal policy can differ from teacher demonstrations \citep{ng2000algorithms}, which may even be sub-optimal itself \citep{brown2019extrapolating,chen2021learning}, and pure behavioral imitation can result in constant sub-optimality \citep{gao2023scaling, zhu2023principled}.

\begin{remark}
    This work departs from \citet[Section~4.1]{lopez2015unifying} essentially in several ways. First, our results hold for any legal data generating distribution $\rho\times\pi^\star$ and do not need to assume data$\to$teacher, data$\to$student, teacher$\to$student transfer rates manually. Second, their result crucially hinges on the assumption that the teacher learns from data faster than the student, while we have clarified our formulation of $\pi^\star$ that differs completely. Finally, their transfer speed from the teacher to the student relies on the hardness of classification of each data point, e.g., the notion of \emph{linear separability} \citep{shalev2014understanding}, yet we consider the product space of finite domains $\cS\times\cA$, which does not have these issues.
\end{remark}

\section{Missing Notation, Definitions, and Derivations}
\label{sec:appendix:missing}

\paragraph{Additional Notation in Appendix.} For any event $\cI$, $\cI^c$ denotes its complementary event. $[K] \coloneqq \{1,\dots,K\}$ and $\overline{[K]} \coloneqq \{0,\dots, K-1\}$ for any positive integer $K$. We set $(S, A) = (|\cS|, |\cA|)$ and index the input and label spaces by $(\cS, \cA) = (\overline{[S]}, [A])$ when necessary.  We also use a general notion of Dirac distribution in proving the lower bounds: given any measurable space $(\cZ, \Sigma)$, we define $\mathsf{Dirac}(\cZ, z)(D) \coloneqq \ind\{z\in D\}, \forall D\in\Sigma$. We denote by $\log$ the natural logarithm and adopt the convention $0\log 0 = 0$.

\begin{definition}\label{def:missingMass}
    For $n$ \emph{i.i.d.} samples $X^n$ drawn from a distribution $\nu$ over an alphabet $\cX$, the number of occurrences of $x$ is denoted by $\nm{x}{X^n} \coloneqq \sum_{i=1}^n\ind\{ X_i = x \}$, upon which we further measure the portion of $\cX$ never observed in $X^n$ by the missing mass
    \begin{equation}\label{eq:MissingMassNotation}
        \mm{\nu}{X^n} \coloneqq \sum_{x\in\cS}\nu(x)\ind\{\nm{x}{X^n} = 0\}.
    \end{equation}
\end{definition}

It is worth mentioning that, $\cD$, $\cS(\cD)$, $\cA(\cD)$, and $\cA(\cD, s), \forall s\in\cS$ are treated as \emph{multisets} when fed into functionals like $\mm{\nu}{\cdot}$, $\nm{x}{\cdot}$, or the cardinality operator $|\cdot|$, for example, $|\cA(\cD, s)| = \nm{s}{\cS(\cD)}$; while in set operations like $\cS\backslash\cS(\cD)$ or under the summation sign like $\sum_{s\in\cS(\cD)}$, where they act as ranges of enumeration, we slightly abuse the notations for simplicity to refer to their \emph{deduplicated} counterparts.

\begin{definition}\label{def:associatedActions}
All the $\nm{s}{\cS(\cD)}$ labels in $\cD$ mapped from $s_i = s$ form a multiset
$$
\cA(\cD, s) \coloneqq \left\{a\in\cA: \text{for all } (x, a)\in\cD \text{ s.t. } x = s\right\}.
$$    
\end{definition}

\subsection{Maximum Likelihood Estimation}\label{subsec:appendix:nllSolDerive}

 \begin{align}\label{eq:mleLoss}
    \nll &\in \argmin_{\pi\in\Delta(\cA|\cS)} \lossCEsingle(\cD) \notag\\
    &= \argmin_{\pi\in\Delta(\cA|\cS)} \sum_{i=1}^n \lossCEsingle(s_i, a_i; \pi) = \argmin_{\pi\in\Delta(\cA|\cS)} -\sum_{s\in\cS}\sum_{a\in\cA} \nm{(s,a)}{\cD}\log(a|s) \notag \\
    &= \argmin_{\pi\in\Delta(\cA|\cS)} - \sum_{s\in\cS(\cD)}\nm{s}{\cS(\cD)}\underbrace{\sum_{a\in\cA}\frac{\nm{(s,a)}{\cD}}{\nm{s}{\cS(\cD)}}\log \pi(a | s)}_{-\lossCEfull\left( \nm{(s,\cdot)}{\cD} / \nm{s}{\cS(\cD)}\Vert \pi(\cdot| s) \right)}.
\end{align}
Noticing that \eqref{eq:mleLoss} is the summation of the cross-entropy between $\nicefrac{\nm{(s,\cdot)}{\cD}}{\nm{s}{\cS(\cD)}}$ and $\pi(\cdot| s)$ weighted by $\nm{s}{\cS(\cD)}$  over $\cS(\cD)$, we figure out the explicit solution of $\nll$ as
\begin{equation*}
    \nll(a | s) \begin{cases}
        = \nm{(s,a)}{\cD} / \nm{s}{\cS(\cD)}, &s \in \cS(\cD),\\
        \in\Delta(\cA)\text{ arbitrarily}, &\text{otherwise.}
    \end{cases}
\end{equation*}

\subsection{Empirical Cross-Entropy Loss}\label{subsec:appendix:empceSolDerive}

\begin{align*}
    \empce \in \argmin_{\pi\in\Delta(\cA|\cS)} \lossEMPCE(\cD, \cR) = \argmin_{\pi\in\Delta(\cA|\cS)} - \sum_{s\in\cS(\cD)}Z_s \sum_{a\in\cA} \frac{\nm{(s, a)}{\cD}\pi^\star(a|s)}{Z_s}\log \pi(a|s),
\end{align*}
where $Z_s \coloneqq \sum_{a\in\cA} \nm{(s, a)}{\cD}\pi^\star(a|s)$. Therefore, by the same cross-entropy minimization argument, the explicit solution is
\begin{equation*}
    \empce(a | s) \begin{cases}
        =\left.
        \begin{cases}
         \nm{(s, a)}{\cD}\pi^\star(a|s) / Z_s, &(s,a) \in \cD,\\
         0, & s\in\cS(\cD), a\notin\cA(\cD, s),
    \end{cases}\right\rbrace &\propto \nm{(s, a)}{\cD}\pi^\star(a|s);\\
        \in\Delta(\cA)\text{ arbitrarily}, &s\notin\cS(\cD).
    \end{cases}
\end{equation*}

\section{Information-Theoretic Arguments}\label{sec:appendix:lbs}

\subsection{Additional Notation in Appendix~\ref{sec:appendix:lbs}}

We write $\condkl{\pi}{\pi^\prime}{\lambda} \coloneqq \EE_{x\sim\lambda}\left[\kl{\pi(\cdot|x)}{\pi^\prime(\cdot|x)}\right]$ for $\lambda\in\Delta(\cS)$ and $\pi,\pi^\prime\in\Delta(\cA|\cS)$ likewise for alphabets $\cS,\cA$ to have notations concise. The values of $\triangle$ in the proofs of Theorem~\ref{thm:lb1}, Theorem~\ref{thm:lb2}, and Theorem~\ref{thm:lb3} are different under the same notation. A similar logic applies to the values of $\pi_\tau$ in the proofs of Theorem~\ref{thm:lb1} and Theorem~\ref{thm:lb2}.

\subsection{Proof of Theorem~\ref{thm:lb1}}\label{sec:appendix:lb1proof}

\begin{proof}
    We fix $\rho$ to be $\mathsf{Uniform}(\cS)$ and define a loss function over $\Delta(\cA|\cS)\times\Delta(\cA|\cS)$ as
    \begin{equation}\label{eq:lbLoss}
        l(\pi, \pi^\prime) \coloneqq \condtv{\pi}{\pi^\prime}{\rho} = \frac{1}{S}\sum_{s=0}^{S-1}\tv{\pi(\cdot|s)}{\pi^\prime(\cdot|s)}.
    \end{equation}
    Obviously, $\left(\Delta(\cA|\cS), l\right)$ becomes a metric space. Without loss of generality, suppose $A$ is even, we decompose $l$ as
    \begin{align}
       d_{sA/2+j}(\pi,\pi^\prime) &\coloneqq l_{s, j}(\pi,\pi^\prime) \coloneqq \frac{|\pi(2j - 1|s) - \pi^\prime(2j - 1|s)| + |\pi(2j|s) - \pi^\prime(2j|s)|}{2S}, \\
        l(\pi,\pi^\prime) &= \sum_{s=0}^{S - 1}\sum_{j=1}^{A/2}l_{s,j}(\pi,\pi^\prime) = \sum_{i=1}^{SA/2}d_i(\pi,\pi^\prime).\label{eq:lbLossDecomp}
    \end{align}
    Inspired by Paninski's construction \citep{paninski2008coincidence}, we define $\pi_\tau$ as
    \begin{equation}\label{eq:lbConstruction}
        \pi_\tau(2j - 1|s) = \frac{1 + \tau_{sA/2+j}\triangle}{A}, \pi_\tau(2j|s) = \frac{1 - \tau_{sA/2+j}\triangle}{A}, \forall (s,j)\in\overline{[S]}\times\left[\frac{A}{2}\right];
    \end{equation}
    where $\tau\in\{-1, +1\}^{AS/2}$ and $\triangle$ is to be specified later. For any $\tau\sim_{i}\tau^\prime$, i.e., any pair in $\{-1, +1\}^{AS/2}$ that differs only in the $i$-th coordinate, the construction of $\pi_\tau$ leads to
    \begin{equation}\label{eq:lbSeparation}
        d_i(\pi_\tau, \pi_{\tau^\prime}) = \frac{2\triangle}{SA}.
    \end{equation}
    We thereby refer $\tau\sim\tau^\prime$ to any pair in $\{-1, +1\}^{AS/2}$ that differs only in one coordinate and obtain
    \begin{equation}
    \begin{aligned}[b]
        &\quad\text{LHS of \eqref{eq:lbSentence}} \geq \inf_{\hat\pi \in \hat\Pi(\cD)} \sup_{\substack{\tau \in\{-1, +1\}^{AS/2}\\ \rho=\mathsf{Uniform}(\cS)}} \EE_{(\rho\times \pi_\tau)^n}l(\hat\pi,\pi_\tau)\\
        &\geq \frac{AS}{2}\cdot\frac{2\triangle/(SA)}{2} \min_{\tau\sim\tau^\prime}\left(1  - \tv{(\rho\times\pi_\tau)^n}{(\rho\times\pi_{\tau^\prime})^n} \right) \\
        &\geq \frac{\triangle}{4} \min_{\tau\sim\tau^\prime} \exp\left( -\kl{(\rho\times\pi_\tau)^n}{(\rho\times\pi_{\tau^\prime})^n} \right) \\
        &= \frac{\triangle}{4} \min_{\tau\sim\tau^\prime} \exp\left( -n\kl{\rho\times\pi_\tau}{\rho\times\pi_{\tau^\prime}} \right) = \frac{\triangle}{4} \min_{\tau\sim\tau^\prime} \exp\left( -n\condkl{\pi_\tau}{\pi_{\tau^\prime}}{\rho} \right) \\
        &= \frac{\triangle}{4} \exp\left( -\frac{n}{SA}\cdot2\triangle\log\frac{1+\triangle}{1-\triangle} \right)\\
        &\geq \frac{\triangle}{4} \exp\left(-8\frac{n}{SA}\triangle^2\right),\label{eq:lbMidRes}
    \end{aligned}        
    \end{equation}
    where
    \begin{itemize}
        \item the second inequality is by Assouad's lemma \citep[Lemma~2]{yu1997assouad},
        \item the third inequality holds due to a variant (Lemma~\ref{lem:BretagnolleHuberIneq}) of the Bretagnolle–Huber inequality,
        \item the first equality holds due to the decomposable property of $\mathsf{KL}$ \citep[Section~2.4]{tsybakov2009nonpara},
        \item the second equality follows from a basic property of $f$-divergence \citep[Proposition~7.2.4]{polyanskiy2022information},
        \item the last equality is by the definition of $\pi_\tau$ in \eqref{eq:lbConstruction} and $\rho = \mathsf{Uniform}(\cS)$,
        \item the last inequality derives from $\log(1 + x) \leq x, x > 0$ and an additional constraint $\triangle \leq 0.5$ we impose.
    \end{itemize}
    The assignment $\triangle = 0.25\sqrt{\nicefrac{SA}{n}}$ with $n \geq \nicefrac{SA}{4}$ is a feasible choice for inequality \eqref{eq:lbMidRes} to hold, whose RHS further equals to
    $$
    \frac{\exp(-0.5)}{16}\sqrt{\frac{SA}{n}}.
    $$
\end{proof}

\subsection{Proof of Theorem~\ref{thm:lb2}}\label{sec:appendix:lb2proof}

\begin{proof}
    Without loss of generality, we assume $A > 2$ is odd, then for a fixed $\triangle \coloneqq 1$, which does NOT vary with $S$, $A$, or $n$ in THIS proof, we define $\Pi \coloneqq \left\{ \pi_\tau: \tau\in\{-1, +1\}^{\nicefrac{AS}{2}} \right\}$, where
    \begin{equation}
    \left.\begin{aligned}
        \pi_\tau(2j - 1|s) &\coloneqq \frac{S(1 + \tau_{sA/2+j}\triangle)}{2(n+1)},\\ \pi_\tau(2j|s) &\coloneqq \frac{S(1 - \tau_{sA/2+j}\triangle)}{2(n+1)};\\
        \pi_\tau(A|s) &\coloneqq 1 - \frac{S}{2}\cdot\frac{A-1}{n+1}.
    \end{aligned}\right\rbrace (s,j)\in\overline{[S]}\times\left[\frac{A-1}{2}\right].
    \end{equation}
    To get a lower bound on a Bayes risk, we design a prior over $\cP$ as
    \begin{equation}\label{eq:lb2:Prior}
        \Lambda \coloneqq \mathsf{Dirac}\left(\Delta(\cS), \mathsf{Uniform}(\cS)\right)\times \Gamma,
    \end{equation}
    where
    $$
    \Gamma \coloneqq \mathsf{Uniform}(\Pi).
    $$
    Intuitively speaking, for any $\rho\times\pi$ sampled from $\Lambda$, $\rho$ must be uniform over $\cS$ and $\pi$ must be some $\pi_\tau$ with $\tau$ uniformly distributed over $\{-1, +1\}^{SA/2}$.\footnote{To be technically rigorous, we say the prior for $\rho$ over $\Delta(\cS)$ and the prior for $\pi^\star$ over $\Delta(\cA|\cS)$ are assigned \emph{independently}, similar assignment also appears in the proof of Theorem~\ref{thm:lb3}.} Therefore, if we let $\Lambda(\cD, \cR)$ be the corresponding posterior over $\cP$ conditioned on $(\cD, \cR)$ and $\Gamma(\cD, \cR)$ be the marginal posterior over $\Delta(\cA|\cS)$, we can by the definition of $\Gamma$ and $\pi_\tau$ obtain $\Lambda(\cD, \cR) = \mathsf{Dirac}\left(\Delta(\cS), \mathsf{Uniform}(\cS)\right)\times \Gamma(\cD, \cR)$ where for any $(s, a)\in\{0,\dots,S-1\}\times\{1, \dots, A-1\}$ and $\pi\sim \Gamma(\cD, \cR)$, by the Bayes rule,
    \begin{equation}\label{eq:lb2:post}
    \begin{cases}
        &\Gamma(\cD, \cR)\left[ \pi(a|s) = \pi^\star(a|s) \right] = 1, (s, a)\in\cD\text{ or } (s, \text{Buddy}(a))\in\cD; \\
        &\Gamma(\cD, \cR)\left[ \pi(a|s) = \frac{S(1 + \triangle)}{2(n+1)} \right] = \Gamma(\cD, \cR)\left[ \pi(a|s) = \frac{S(1 - \triangle)}{2(n+1)} \right] = \frac{1}{2}, \text{otherwise};
    \end{cases}
    \end{equation}
    where (recall that $A > 2$ is assumed to be odd) we define the ``Buddy'' for $a\in[A-1] = \{1, \dots, A-1\}$ as
    \begin{equation}
        \text{Buddy}(a) \coloneqq \begin{cases}
            a - 1, &a\text{ is even};\\
            a + 1, &a\text{ is odd}.
        \end{cases}
    \end{equation}
    The intuition behind \eqref{eq:lb2:post} is that if $\pi\sim\Gamma(\cD, \cR)$, the marginal posterior of $\pi(a|s)$ for any seen $(s, a)$ in $\cD$ must be a Dirac concentrated at $\pi^\star(a|s)$ and by the design of $\Pi = \{\pi_\tau\}$, $\pi(a|s)$ is also determined if $(s, \text{Buddy}(a)) \in\cD$.\footnote{Rigorously speaking, $\pi^\star(a|s)$ in \eqref{eq:lb2:post} refers to some concrete realization to some $\pi_\tau(a|s)$, which is collected by $\cR$. It is a slight abuse of notation and a similar one appears in \eqref{eq:lb3:post} in the proof of Theorem~\ref{thm:lb3}, too.} Note that the last label, $A$, is designed to be ignored, which we do not consider in both \eqref{eq:lb2:post} and the following argument driven by Fubini's theorem. We define a event $\cE_\cD(s, a)$ for every $(s, a)$ with $a < A$ as
    $$
    \cE_\cD(s, a) = (s, a)\in\cD\text{ or } (s, \text{Buddy}(a))\in\cD.
    $$
    Next, we can apply Fubini's theorem to the Bayes risk with $\Lambda$ as its prior.
    {\allowdisplaybreaks
    \begin{align}
        &\quad\EE_{ \rho\times\pi^\star \sim \Lambda }\left[ \EE_{(\rho\times \pi^\star)^n}\condtv{\hat\pi}{\pi^\star}{\rho} \right] \notag\\
        &= \frac{1}{S}\sum_{s\in\cS}\EE_{\pi^\star \sim \Gamma}\EE_{\cD\sim (\rho\times\pi^\star)^n}\EE_{\pi\sim\Gamma(\cD, \cQ)}\tv{\hat\pi(\cdot|s)}{\pi(\cdot|s)} \notag\\
        &\geq \frac{1}{2S}\sum_{s\in\cS}\EE_{\pi^\star \sim \Gamma}\EE_{\cD\sim (\rho\times\pi^\star)^n}\sum_{a<A}\EE_{\pi\sim\Gamma(\cD, \cQ)}\abs{ \hat\pi(a|s) - \pi(a|s) } \notag\\
        &= \frac{1}{2S}\sum_{s\in\cS}\EE_{\pi^\star \sim \Gamma}\EE_{\cD\sim (\rho\times\pi^\star)^n}\sum_{a<A}\left\{\right. \notag\\
        &\quad \EE_{\pi\sim\Gamma(\cD, \cQ)}\left[\abs{ \hat\pi(a|s) - \pi(a|s) } |\cE_\cD(s, a) \right]\EE[\ind\{\cE_\cD(s, a)\} | \cD] \notag\\
        &+ \EE_{\pi\sim\Gamma(\cD, \cQ)}\left[\abs{ \hat\pi(a|s) - \pi(a|s) } | \cE^c_\cD(s, a) \right]\EE[\ind\{\cE^c_\cD(s, a)\} | \cD] \notag\\
        &\left.\right\} \notag\\
        &\geq \frac{1}{2S}\sum_{s\in\cS}\EE_{\pi^\star \sim \Gamma}\EE_{\cD\sim (\rho\times\pi^\star)^n}\sum_{a<A} \EE_{\pi\sim\Gamma(\cD, \cQ)}\left[\abs{ \hat\pi(a|s) - \pi(a|s) } | \cE^c_\cD(s, a) \right]\EE[\ind\{\cE^c_\cD(s, a)\} | \cD] \notag\\
        &= \frac{1}{2S}\sum_{s\in\cS}\EE_{\pi^\star \sim \Gamma}\EE_{\cD\sim (\rho\times\pi^\star)^n}\sum_{a<A}\EE[\ind\{\cE^c_\cD(s, a)\} | \cD]\cdot \left\{\right. \notag\\
        &\quad \frac{1}{2} \abs{ \hat\pi(a|s) - \frac{S(1 + \triangle)}{2(n+1)} } + \frac{1}{2} \abs{ \hat\pi(a|s) - \frac{S(1 - \triangle)}{2(n+1)} } \notag\\
        \left.\right\} \notag\\
        &\geq \frac{1}{2S}\sum_{s\in\cS}\EE_{\pi^\star \sim \Gamma}\EE_{\cD\sim (\rho\times\pi^\star)^n} \sum_{a < A} \EE[\ind\{\cE^c_\cD(s, a)\} | \cD] \frac{1}{2}\cdot \frac{S}{n+1} \label{eq:lb3:keyReduction:TriangleEquals1}\\
        &= \frac{1}{4(n+1)}\sum_{s, a:a<A} \PP\left\{(s, a)\notin\cD\text{ and } (s, \text{Buddy}(a))\notin\cD\right\} \label{eq:lb3:keyReduction:priorIndepenentProb} \\
        &= \frac{1}{4(n+1)}\sum_{s, a:a<A} \left( 1 - \rho(s)\cdot \frac{S}{n+1} \right)^n = \frac{1}{4(n+1)}\sum_{s, a:a<A} \left( 1 - \frac{1}{S}\cdot \frac{S}{n+1} \right)^n \notag\\
        &= \frac{1}{4(n+1)}\sum_{s, a:a<A} \left( 1 - \frac{1}{n+1} \right)^n \notag\\
        &\geq \frac{S(A-1)}{4e(n+1)} \gtrsim \frac{SA}{n},
    \end{align} }
    where the inequality in \eqref{eq:lb3:keyReduction:TriangleEquals1} holds because of the triangle inequality and $\triangle = 1$ by design, and the equality in \eqref{eq:lb3:keyReduction:priorIndepenentProb} holds because for any possible value of $\pi^\star$ a priori, i.e., any $\pi_\tau$, $\PP\left\{\cE^c_\cD(s, a)\right\}$ is always $\left(1 - \nicefrac{1}{(n+1)}\right)^n$ by the design of $\Pi = \{\pi_\tau\}$; the penultimate and the last equality hold due to the fact that the distribution of $\rho$ is always $\mathsf{Dirac}\left(\Delta(\cS), \mathsf{Uniform}(\cS)\right)$ no matter whether a priori or a posteriori (conditioned on $(\cD, \cR)$). Since the minimax risk is bounded from below by the worst-case Bayes risk \citep[Theorem~28.1]{polyanskiy2022information}, the proof is completed.
\end{proof}

\subsection{Proof of Theorem~\ref{thm:lb3}}\label{sec:appendix:lb3proof}

\begin{proof}
 We assign $\xi = \nicefrac{1}{(n+1)}$ and design a $p\in\Delta(\cS)$ as $p(0) = 1 - \nicefrac{S - 1}{(n+1)}$ and $p = \nicefrac{1}{(n+1)}$ for all other inputs following \citet[Figure~1 (b)]{rajaraman2020toward}. Then it suffices to get a lower bound for a Bayes risk given a prior over $\cP$, which we design as
    \begin{equation}\label{eq:lb3:Prior}
        \Lambda_3 \coloneqq \mathsf{Dirac}\left(\Delta(\cS), p\right)\times\Gamma_3,
    \end{equation}
    where 
    $$
    \Gamma_3 \coloneqq \mathsf{Uniform}(\Pi_\text{det}), \Pi_\text{det} \coloneqq \left\{\pi\in\Delta(\cA|\cS): \pi(\cdot|s)\in\mathsf{Dirac}\left(\cA\right) ,\forall s\in\cS\right\}.
    $$
    Intuitively speaking, $\pi^\star$ is uniformly distributed over all deterministic policies, which indicates the marginal prior distribution of $\pi^\star(\cdot|s)$ for any $s\in\cS$ is
    $$
    \Gamma_3\left[\pi(\cdot|s) = \mathsf{Dirac}(\cA, a)\right] = \frac{1}{A}, \forall a\in\cA.
    $$
    We abbreviate the corresponding posterior of $\Lambda_3$ (resp. $\Gamma_3$) conditioned on $(\cD, \cQ)$ as $\Lambda_3(\cD, \cQ)$ (resp. $\Gamma_3(\cD, \cQ)$), which by definition implies $\Lambda_3(\cD, \cQ) = \mathsf{Dirac}(\Delta(\cS), p)\times\Gamma_3(\cD, \cQ)$ and by the Bayes formula implies that for any $s\in\cS$ and $\pi\sim \Gamma_3(\cD, \cQ)$,
    \begin{equation}\label{eq:lb3:post}
        \begin{cases}
        \Gamma_3(\cD, \cQ)\left[ \pi(\cdot|s) = \pi^\star(\cdot|s) \right] &= 1, s\in\cS(\cD); \\
        \Gamma_3(\cD, \cQ)\left[ \pi(\cdot|s) = \mathsf{Dirac}(\cA, a) \right] &= 1/A, s\in\cS\backslash\cS(\cD).
    \end{cases}
    \end{equation}
    Without loss of generality, we assume $A > 1$ is even and then by Fubini's theorem, 
    {\allowdisplaybreaks
    \begin{align}
        &\quad\EE_{ \rho\times\pi^\star \sim \Lambda_3 }\left[ \EE_{(\rho\times \pi^\star)^n}\condtv{\hat\pi}{\pi^\star}{\rho} \right]  \notag\\
        &= \sum_{s\in\cS}\rho(s)\EE_{\pi^\star \sim \Gamma_3}\EE_{\cD\sim (\rho\times\pi^\star)^n}\EE_{\pi\sim\Gamma_3(\cD, \cQ)}\tv{\hat\pi(\cdot|s)}{\pi(\cdot|s)} \notag\\
        &= \sum_{s\in\cS}\rho(s)\EE_{\pi^\star \sim \Gamma_3}\EE_{\cD\sim (\rho\times\pi^\star)^n} \left\{\EE_{\pi\sim\Gamma_3(\cD, \cQ)}\left[ \tv{\hat\pi(\cdot|s)}{\pi(\cdot|s)} \big| s\in\cS(\cD)\right]\EE\left[\ind\left(s\in\cS(\cD)\right)|\cD\right] \right. \notag \\
        &+\left. \EE_{\pi\sim\Gamma_3(\cD, \cQ)}\left[ \tv{\hat\pi(\cdot|s)}{\pi(\cdot|s)} \big| s\notin\cS(\cD)\right]\EE\left[\ind\left(s\notin\cS(\cD)\right)|\cD\right] \right\} \notag\\
        &\geq \sum_{s\in\cS}\rho(s)\EE_{\pi^\star \sim \Gamma_3}\EE_{\cD\sim (\rho\times\pi^\star)^n}\left\{  \EE\left[\ind\left(s\notin\cS(\cD)\right)|\cD\right] \EE_{\pi\sim\Gamma_3(\cD, \cQ)}\left[ \tv{\hat\pi(\cdot|s)}{\pi(\cdot|s)} \big| s\notin\cS(\cD)\right] \right\} \notag\\
        &= \sum_{s\in\cS}\rho(s)\EE_{\pi^\star \sim \Gamma_3}\EE_{(\rho\times\pi^\star)^n}\left\{ \EE\left[\ind\left(s\notin\cS(\cD)\right)|\cD\right]\cdot  \frac{1}{A}\sum_{a\in\cA}\tv{\hat\pi(\cdot|s)}{\mathsf{Dirac}(\cA, a)} \right\}, \notag\\
        &= \sum_{s\in\cS}\rho(s)\EE_{\pi^\star \sim \Gamma_3}\EE_{(\rho\times\pi^\star)^n}\left\{ \right. \notag\\
        &\left. \EE\left[\ind\left(s\notin\cS(\cD)\right)|\cD\right] \cdot \frac{1}{A}\sum_{a=0}^{A/2 - 1}\left[\tv{\hat\pi(\cdot|s)}{\mathsf{Dirac}(\cA, a)} + \tv{\hat\pi(\cdot|s)}{\mathsf{Dirac}(\cA, a + A/2)} \right] \right. \notag\\
        &\left.\right\} \notag\\
        &\geq \sum_{s\in\cS}\rho(s)\PP\left(s\notin\cS(\cD)\right)\frac{1}{A}\sum_{a=0}^{A/2 - 1}\tv{\mathsf{Dirac}(\cA, a)}{\mathsf{Dirac}(\cA, a + A/2)} 
 \notag\\
        &= \sum_{s\in\cS}\rho(s)\PP\left(s\notin\cS(\cD)\right)\frac{1}{A}\cdot\frac{A}{2} = \frac{1}{2}\sum_{s\in\cS}\rho(s)\PP\left(s\notin\cS(\cD)\right),\label{eq:lb3:mid}
    \end{align} }
    where the last inequality holds due to the triangle inequality of $\mathsf{TV}$. Therefore,
    \begin{equation}
    \begin{aligned}[b]
        \text{LHS of \eqref{eq:lb3:mid}} &\geq 0.5\sum_{s\in\cS}\rho(s)\PP\left(s\notin\cS(\cD)\right) = 0.5\sum_{s\in\cS}\rho(s)\left(1 - \rho(s)\right)^n \\
        &\geq \frac{S - 1}{2(n + 1)}\left(1 - \frac{1}{n+1}\right)^n \geq \frac{S - 1}{2e(n+1)} \gtrsim \frac{S}{n},
    \end{aligned}
    \end{equation}
    where the second inequality is by only considering the $S - 1$ inputs with mass $\nicefrac{1}{(n+1)}$. Since the minimax risk is bounded from below by the worst-case Bayes risk, the proof is completed.
\end{proof}

\section{Arguments for Specific Learners}\label{sec:appendix:specificLearners}

\subsection{Additional Definitions in Appendix~\ref{sec:appendix:specificLearners}}

In this section we denote the MLE of $\rho$ by

\begin{equation}\label{eq:rhoMLE}
        \hat\rho(\cdot) \coloneqq \frac{\nm{(\cdot)}{\cS(\cD)}}{n}.
\end{equation}

The event $B_{s, i}$ defined as follows will be used in the proofs of Theorem~\ref{thm:mle} and Theorem~\ref{thm:empselRefined}. 

\begin{equation}\label{eq:eventB}
    B_{s, i} \coloneqq \left\{ \nm{s}{\cS(\cD)} = i \right\}, \forall (s, i)\in \cS\times\overline{[n+1]}.
\end{equation}

\subsection{Proof of Theorem~\ref{thm:mle}}\label{subsec:mleProofs}

\subsubsection{Proof of the High-Probability Bound}

\begin{proof}

 For $|\cS|> 1$, we define
    $$
    u_s \coloneqq \nm{s}{\cS(\cD)}\tv{\nll(\cdot| s)}{\pi^\star(\cdot| s)}, \forall s\in\cS.
    $$
    We decompose the LHS of \eqref{eq:mle:highProbBound} as
    \begin{equation}
        \begin{aligned}[b]\label{eq:mle_decomp}
            \text{LHS} &= \sum_{s\in\cS}\left(\frac{\nm{s}{\cS(\cD)}}{n} + \rho(s) - \hat\rho(s)  \right) \tv{\nll(\cdot| s)}{\pi^\star(\cdot| s)} \\
            &\leq \sum_{s\in\cS} \frac{u_s}{n} + \sum_{s\in\cS}\left| \rho(s) - \hat\rho(s) \right| \tv{\nll(\cdot| s)}{\pi^\star(\cdot| s)} \\
            &\leq \underbrace{ \frac{1}{n}\sum_{s\in\cS} u_s + 2\tv{\rho}{\hat\rho} }_{(i)},
        \end{aligned}
    \end{equation}
    where the first inequality is by triangle inequality and the second one holds due to the boundedness of $\mathsf{TV}$.
    We define another two types of events to bound $(i)$ in \eqref{eq:mle_decomp}:
    \begin{align}
        D_s &\coloneqq \left\{ u_s \leq \sqrt{ \frac{\nm{s}{\cS(\cD)}}{2}\left( |\cA|\log2 + \log\frac{|\cS|+1}{\delta} \right) } \right\}, \forall s\in\cS; \\
        E &\coloneqq \left\{ 2\tv{\rho}{\hat\rho} \leq \sqrt{ \frac{2}{n}\left(|\cS|\log2 + \log\frac{|\cS| + 1}{\delta} \right) } \right\}.
    \end{align}
        Notice that $\PP(D_s^c | B_{s, 0}) = 0$ by the definition of $B_{s,i}$ in \eqref{eq:eventB} and for any $i>0$, $\PP(D_s^c | B_{s, i}) \leq \nicefrac{\delta }{(|\cS| + 1)}, \forall s\in\cS $ by Lemma~\ref{lem:TVempiricalUpperBound}; thus by the law of total probability,
    \begin{equation}
        \PP(D_s) = \sum_{i=0}^n\PP(D_s | B_{s, i})\PP(B_{s, i}) \geq \left( 1 - \frac{\delta}{|\cS| + 1}\right)\sum_{i=0}^n\PP(B_{s, i}) = 1 - \frac{\delta}{|\cS| + 1}, \forall s\in\cS.
    \end{equation}
        Also noticing that $\PP(E^c) \leq \nicefrac{\delta }{ (|\cS| + 1)}$ by Lemma~\ref{lem:TVempiricalUpperBound}, we apply a union bound over $E^c$ and $\{D_s^c\}_{s\in\cS}$ for $(i)$ in \eqref{eq:mle_decomp} to conclude that with probability at least $1 - \delta$,
    \begin{equation}\label{eq:highPfull_1}
        (i) \leq \sqrt{\frac{|\cA|\log2 + \log((|\cS| + 1)/\delta)}{2}}\cdot \underbrace{ \frac{\sum_{s\in\cS}\sqrt{\nm{s}{\cS(\cD)}}}{n} }_{\heartsuit}  + \sqrt{ \frac{2}{n}\left(|\cS|\log2 + \log\frac{|\cS| + 1}{\delta} \right) }.
    \end{equation}
    By the Cauchy-Schwarz inequality,
    \begin{equation}\label{eq:sum_ns_1}
        \heartsuit\text{ in \eqref{eq:highPfull_1}} \leq \frac{1}{n}\sqrt{|\cS|\sum_{s\in\cS}\nm{s}{\cS(\cD)}} = \sqrt{\frac{|\cS|}{n}}.
    \end{equation}
    Substituting \eqref{eq:sum_ns_1} back to the RHS of \eqref{eq:highPfull_1} yields the conclusion. The case of $|\cS| = 1$ follows from Lemma~\ref{lem:TVempiricalUpperBound}.
\end{proof}

\subsubsection{Proof of the Worst-Case Upper Bound in Expectation}

\begin{proof}
    Taking expectation on both sides of \eqref{eq:mle_decomp} yields
    \begin{equation}\label{eq:mleExpBound_decomp}
    \begin{aligned}[b]
        &\EE\left[ \condtv{\nll}{\pi^\star}{\rho} \right] \leq {\frac{1}{n}\sum_{s\in\cS}\EE u_s} + \sqrt{\frac{|\cS|}{n}} \\
        =& {\frac{1}{n}\sum_{s\in\cS}\EE \left[\underbrace{\nm{s}{\cS(\cD)}\EE[\tv{\nll(\cdot|s)}{\pi^\star(\cdot|s)}|\nm{s}{\cS(\cD)}]}_{\eqqcolon\tilde{u}_s}\right]} + \sqrt{\frac{|\cS|}{n}},
    \end{aligned}
    \end{equation}
    where the inequality holds due to Lemma~\ref{lem:TVempiricalExpectationBound}. For every $s$, we trivially have
    \begin{equation}\label{eq:nllExpBound:nullCase}
        \tilde{u}_s \leq \frac{\sqrt{|\cA|\nm{s}{\cS(\cD)}}}{2}
    \end{equation}
    if conditioned on $B_{s,0}$. If otherwise conditioned on $B_{s, 0}^c$, we still have
    \begin{equation}\label{eq:nllExpBound:otherwise}
        \tilde{u}_s \leq \frac{\sqrt{|\cA|\nm{s}{\cS(\cD)}}}{2},
    \end{equation}
    where the inequality follows from Lemma~\ref{lem:TVempiricalExpectationBound}. Therefore, substituting \eqref{eq:nllExpBound:nullCase} and \eqref{eq:nllExpBound:otherwise} back to \eqref{eq:mleExpBound_decomp} gives
    \begin{align*}
        \text{LHS of \eqref{eq:mleExpBound_decomp}} &\leq \frac{\sqrt{|\cA|}}{2n}\sum_{s\in\cS}\EE\sqrt{\nm{s}{\cS(\cD)}}  + \sqrt{\frac{|\cS|}{n}} \leq \sqrt\frac{{|\cA|}}{2n}\sum_{s\in\cS}\sqrt{\rho(s)} + \sqrt{\frac{|\cS|}{n}} \\
        &\lesssim \sqrt{\frac{|\cS||\cA|}{n}},
    \end{align*}
    where the first inequality follows from the law of total expectation with respect to $B_{s, 0}$ and $B_{s,0}^c$, the second inequality follows from Jensen's inequality together with the definition of $\nm{s}{\cS(\cD)}$, and the last inequality is by the Cauchy-Schwarz inequality.
\end{proof}

\subsubsection{Proof of the Instance-Depedent Upper Bound in Expectation}

\begin{proof}
    We define the set of the numbers of occurences of all inputs as
    \begin{equation}\label{eq:NS}
        N_\cS \coloneqq \left\{ \nm{s}{\cS(\cD)}:s\in\cS \right\}.
    \end{equation}
    Then we decompose the LHS of \eqref{eq:mle:instanceDependent} as
    \begin{equation}
        \begin{aligned}[b]\label{eq:instDep_mle_decomp}
            &\quad \sum_{s\in\cS}\rho(s) \EE\left[\tv{\nll(\cdot|s)}{\pi^\star(\cdot|s)}\right] \\
            &= \EE\left[\right.\\
            &\quad\sum_{s\in\cS(\cD)}\rho(s)\EE\left[\tv{\nll(\cdot|s)}{\pi^\star(\cdot|s)}\bigg| N_\cS\right] \\
            &+ \sum_{s\in\cS\backslash\cS(\cD)}\rho(s)\EE\left[\tv{\nll(\cdot|s)}{\pi^\star(\cdot|s)}\bigg| N_\cS\right]\\
            &\quad\left.\right] \\
            &\leq \underbrace{\EE\left[ \sum_{s\in\cS(\cD)}\rho(s)\EE\left[\tv{\nll(\cdot|s)}{\pi^\star(\cdot|s)}\bigg| \nm{s}{\cS(\cD)}\right] \right]}_{I_1} + \underbrace{\EE\mm{\rho}{\cS(\cD)}}_{I_2},
        \end{aligned}
    \end{equation}
    where the inequality is by the definition and boundedness of $\tv{\nll(\cdot|s)}{\pi^\star(\cdot|s)}$. We divide $I_1$ and $I_2$ so at to conquer them as follows.
    \paragraph{Bounding $I_1$.} For every $s\in\cS$, we define
    $$
    I_1(s)\coloneqq\EE\left[\tv{\nll(\cdot|s)}{\pi^\star(\cdot|s)}\bigg| \nm{s}{\cS(\cD)}\right].
    $$
    Then we can bound $I_1(s)$ by Jensen's inequality for $s\in\cS(\cD)$:
    \begin{equation}
        \begin{aligned}[b]\label{eq:I1sBound}
            I_1(s) &= \frac{1}{2}\sum_{a\in\cA}\EE\left[\abs{ \nll(a|s) - \pi^\star(a|s) }\bigg| \nm{s}{\cS(\cD)}\right] \\
            &\leq \frac{1}{2}\sum_{a\in\cA}\sqrt{\EE\left[\frac{\left({ \nm{s}{\cS(\cD)}\nll(a|s) - \nm{s}{\cS(\cD)}\pi^\star(a|s) }\right)^2}{[\nm{s}{\cS(\cD)}]^2}\bigg| \nm{s}{\cS(\cD)}\right]} \\
            &= \frac{1}{2}\sum_{a\in\cA}\sqrt{\frac{\pi^\star(a|s)\left(1 - \pi^\star(a|s)\right)}{\nm{s}{\cS(\cD)}}},
        \end{aligned}
    \end{equation}
    where the last equality holds due to the observation that
    $$
    \nm{s}{\cS(\cD)}\nll(a|s)\big|\nm{s}{\cS(\cD)}\sim\mathsf{Binomial}\left(\nm{s}{\cS(\cD)}, \pi^\star(a|s)\right).
    $$
    Therefore, we can bound the summation inside the expectation of $I_1$ as
    \begin{equation}
        \begin{aligned}[b]\label{eq:insideI1Bound}
            \sum_{s\in\cS(\cD)}\rho(s)I_1(s) &= \frac{1}{2}\sum_{a\in\cA}\sum_{s\in\cS(\cD)}\sqrt{\rho(s)}\sqrt{\rho(s)\frac{\pi^\star(a|s)\left(1 - \pi^\star(a|s)\right)}{\nm{s}{\cS(\cD)}}} \\
            &\leq \frac{1}{\sqrt{2}}\sum_{a\in\cA}\sum_{s\in\cS}\sqrt{\rho(s)}\sqrt{\rho(s)\frac{\pi^\star(a|s)\left(1 - \pi^\star(a|s)\right)}{1 + \nm{s}{\cS(\cD)}}} \\
            &\leq \frac{1}{\sqrt{2}}\sum_{a\in\cA}\sqrt{\sum_{s\in\cS}\rho(s)\frac{\pi^\star(a|s)\left(1 - \pi^\star(a|s)\right)}{1 + \nm{s}{\cS(\cD)}}},
        \end{aligned}
    \end{equation}
    where the first inequality holds due to $\nm{s}{\cS(\cD)} \geq 1, \forall s\in\cS(\cD)$ and the last inequality is by the Cauchy-Schwarz inequality.
    Substituting \eqref{eq:insideI1Bound} into $I_1 = \EE[\sum_{s\in\cS(\cD)}\rho(s)I_1(s)]$ gives
    \begin{equation}
        \begin{aligned}[b]\label{eq:I1coarseBound}
            I_1 &\leq \frac{1}{\sqrt{2}}\sum_{a\in\cA}\sqrt{\sum_{s\in\cS}\pi^\star(a|s)\left(1 - \pi^\star(a|s)\right)\EE\frac{\rho(s)}{1 + \nm{s}{\cS(\cD)}}} \\
            &\leq \frac{1}{\sqrt{2}}\sum_{a\in\cA}\sqrt{\sum_{s\in\cS}\frac{\pi^\star(a|s)\left(1 - \pi^\star(a|s)\right)}{n+1}} \\
            &\leq \frac{1}{\sqrt{2(n+1)}}\sum_{a\in\cA}\sqrt{\sum_{s\in\cS}\min\left(\pi^\star(a|s), 1 - \pi^\star(a|s)\right)} \\
            &\leq \frac{1}{\sqrt{2(n+1)}}\sqrt{|\cA|\sum_{a\in\cA}\sum_{s\in\cS}\min\left(\pi^\star(a|s), 1 - \pi^\star(a|s)\right)} \\
            &\leq \sqrt{\frac{|\cS||\cA|}{2(n+1)}}\cdot\sqrt{ \underbrace{\max_{s\in\cS}\sum_{a\in\cA}\min\left(\pi^\star(a|s), 1 - \pi^\star(a|s)\right)}_{\tilde{\xi}(\pi^\star)} },
        \end{aligned}
    \end{equation}
    where the first inequality is by Jensen's inequality, the second inequality derives from Lemma~\ref{lem:BinInvUpperBound}, and the penultimate inequality holds due to the Cauchy-Schwarz inequality. $\tilde{\xi}(\pi^\star)$ in \eqref{eq:I1coarseBound} can be further bounded from above by
    \begin{align*}
        &\quad\max_{s\in\cS}\min_{b\in\cA}\left( 1 - \pi^\star(b|s) + \sum_{a:a\neq b}\pi^\star(a|s) \right) = \max_{s\in\cS}\min_{b\in\cA}\tv{\pi^\star(\cdot|s)}{\mathsf{Dirac}(\cA, b)}\\
        &= \max_{s\in\cS}\dist{\mathsf{TV}}{\pi^\star(\cdot|s)}{\mathsf{Dirac}(\cA)} = \xi(\pi^\star).
    \end{align*}
    To sum up, $I_1 \lesssim \sqrt{\xi(\pi^\star)|\cS||\cA|n^{-1}}$.

    \paragraph{Bounding $I_2$.} Explicit calculation yields
    $$
    I_2 = \sum_{s\in\cS}\rho(s)\left(1 - \rho(s)\right)^n \leq \frac{4|\cS|}{9n} \lesssim \frac{|\cS|}{n},
    $$
    where the inequality follows from Lemma~\ref{lem:visElem}.
\end{proof}

\subsection{Proof of Lemma~\ref{lem:empce_bias}}

\begin{proof}

Since $|\cS| = 1$, we omit the conditioning on $s\in\cS$ for brevity in this proof. By Lemma~\ref{lem:gc4pmf},
    $$
    \nll \stackrel{a.s.}{\longrightarrow} \pi^\star.
    $$
    Therefore, applying the continuous mapping theorem \citep[Theorem~3.2.10]{durrett2019probability} to \eqref{eq:empce_rel_mle} gives
    $$
    \empce(a) \stackrel{a.s.}{\longrightarrow} \frac{\left[\pi^\star(a)\right]^2}{\sum_{b\in\cA}\left[\pi^\star(b)\right]^2}
    $$
    \emph{uniformly} for every $a\in\cA$.
    
\end{proof}

\subsection{Proof of Theorem~\ref{thm:empce_bias}}\label{subsec:Prooff:thm:empce_bias}

\begin{proof}
    By \eqref{eq:mleSol} and \eqref{eq:empceSol}, $\nll(\cdot|s)\propto \nm{s}{\cS(\cD)}$ and $\empce(\cdot|s)\propto \nm{s}{\cS(\cD)}\pi^\star(\cdot|s)$ for any $s\in\cS(\cD)$. Therefore, the solution set of $\empce$ coincides with that of $\nll$ only if $\pi^\star = \mathsf{Uniform}(\cA)$ or $\pi^\star\in\mathsf{Dirac}(\cA)$; otherwise Lemma~\ref{lem:empce_bias} implies that
    $$
    \liminf_{n\to\infty} \condtv{\empce}{\pi^\star}{\rho} > 0\text{ almost surely,}
    $$
    and we thus rigorously justify \emph{the} $\Omega(1)$ in Table~\ref{tab:resultsTeaser} with probability one.
\end{proof}

\subsection{Proof of Theorem~\ref{thm:empselRefined}}

\subsubsection{Proof of the High-Probability Bounds}

\begin{proof}
    For $|\cS| > 1$, we define 
    $$
    v_s \coloneqq \nm{s}{\cS(\cD)}\tv{\empsel(\cdot| s)}{\pi^\star(\cdot| s)}
    $$
    and decompose LHS of \eqref{eq:empselRefined:sentence} into three terms as
        \begin{align}
            \text{LHS} &\leq \sum_{s\in\cS} \hat\rho(s)\tv{\empsel(\cdot| s)}{\pi^\star(\cdot| s)} + \sum_{s\in\cS}\left| \rho(s) - \hat\rho(s) \right| \tv{\empsel(\cdot| s)}{\pi^\star(\cdot| s)} \label{eq:empselRefined_decomp_first}\\
            &\leq \underbrace{ \frac{1}{n}\sum_{s\in\cS} v_s + \frac{1}{n}\sum_{s\in\cS(\cD)}\left| \frac{\rho(s)}{\hat\rho(s)} - 1 \right| v_s + \overbrace{\sum_{s\in\cS\backslash\cS(\cD)}\rho(s)}^{\mm{\rho}{\cS(\cD)}\text{, matches Definition~\ref{def:missingMass}}} }_{(ii)}, \label{eq:empselRefined_decomp}
        \end{align}
    where $\hat\rho$ is defined in \eqref{eq:rhoMLE} and the second inequality follows from $\hat\rho(s) = 0, \forall s\in\cS\backslash\cS(\cD)$ along with the boundedness of $\mathsf{TV}$. We additionally define two kinds of events to bound $(ii)$ in \eqref{eq:empselRefined_decomp}.
    \begin{align*}
        \check D_s &= \left\{ v_s \leq \frac{4}{9}|\cA| + 3\sqrt{|\cA|\log\frac{|\cS| + 2}{\delta}} \right\}, \forall s\in\cS;\\
        \check E &= \left\{ \mm{\rho}{\cS(\cD)} \leq \frac{4|\cS|}{9n} + \frac{3\sqrt{|\cS|}}{n}\log\frac{|\cS| + 2}{\delta} \right\}.
    \end{align*}
    Recall Definition~\ref{def:associatedActions} for $\cA(\cD, s)$, if $\nm{s}{\cS(\cD)} > 0$, by \eqref{eq:empselSol},
    \begin{equation}\label{eq:tvUpperBoundedByMissingMass}
    \begin{aligned}[b]
        &2\tv{\empsel(\cdot| s)}{\pi^\star(\cdot| s)} = \sum_{a\in\cA}\left| \empsel(a| s) - \pi^\star(a| s) \right| \\
        =& \sum_{a\in\cA\backslash\cA(\cD, s)}\left| \empsel(a| s) - \pi^\star(a| s) \right| \leq 2 \mm{\pi^\star(\cdot|s)}{\cA(\cD, s)},
    \end{aligned}
    \end{equation}
    where the inequality holds due to triangle inequality. Consequently,
    $$
    v_s \leq \nm{s}{\cS(\cD)}\mm{\pi^\star(\cdot|s)}{\cA(\cD, s)}, \forall s \in \cS(\cD);
    $$
    to which we apply Lemma~\ref{lem:missingMassCor} to obtain
    $$
    \PP({\check D}_s^c | B_{s, i}) \leq \frac{\delta}{|\cS| + 2}, \forall i > 0;
    $$
    where $B_{s, i}$ is defined in \eqref{eq:eventB}. Also noticing that $\PP({\check D}_s^c | B_{s, 0}) = 0$ by definition, we can control $\PP(\check D_s)$ by
    $$
    \PP(\check D_s) = \sum_{i}^n \PP(\check D_s |B_{s, i})\PP(B_{s, i}) \geq 
 \left(1 - \frac{\delta}{|\cS| + 2}\right)\sum_{i}^n\PP(B_{s, i}) = 1 - \frac{\delta}{|\cS| + 2}, \forall s \in \cS.
    $$
    Since $\PP(\check E) \geq 1 - \nicefrac{\delta}{ (|\cS| + 2) }$ by Lemma~\ref{lem:missingMassCor}, we apply a union bound over ${\check E}^c$ and $\{{\check D}_s^c\}_{s\in\cS}$ for $(ii)$ in \eqref{eq:empselRefined_decomp} to conclude that with probability at least $1 - \nicefrac{(|\cS| + 1)\delta}{(|\cS| + 2)}$,
    \begin{equation}\label{eq:empselDetailedUpperBound}
    \begin{aligned}
        (ii) &\leq \frac{4|\cA| + 3\sqrt{|\cA|\log((|\cS| + 2)/\delta)}}{9{n}} \cdot\left( |\cS| + \overbrace{\sum_{s\in\cS(\cD)} \underbrace{ \left| \frac{\rho(s)}{\hat\rho(s)} - 1 \right| }_{\eqqcolon o_s}}^{\blacktriangle}  \right) \\
        &+ \frac{4|\cS|}{9n} + \frac{3\sqrt{|\cS|}}{n}\log\frac{|\cS| + 2}{\delta}.
    \end{aligned}
    \end{equation}
    We further decompose $\blacktriangle$ in \eqref{eq:empselDetailedUpperBound} in a \emph{pragmatically} tight enough way as
    \begin{equation}\label{eq:TolerantDecomp}
    \begin{aligned}[b]
        \blacktriangle &\leq |\cS| + \sum_{s\in\cS(\cD)}\frac{\rho(s)}{\hat\rho(s)} = |\cS| + \sum_{s\in\cS(\cD)}\frac{n\rho(s)}{\nm{s}{\cS(\cD)}} \\
        &= |\cS| + \sum_{s\in\cS(\cD)}\frac{2n\rho(s)}{\nm{s}{\cS(\cD)} + \nm{s}{\cS(\cD)}} \leq |\cS| + 2\sum_{s\in\cS(\cD)}\frac{n\rho(s)}{\nm{s}{\cS(\cD)} + 1}\\
        &\leq |\cS| + 2\sum_{s\in\cS}\frac{n\rho(s)}{\nm{s}{\cS(\cD)} + 1} \\
        &= |\cS| + \underbrace{\sum_{s\in\bar{\cS}}\frac{n\rho(s)}{\nm{s}{\cS(\cD)} + 1}}_{\bar\blacktriangle} + \overbrace{\sum_{s\in\tilde\cS}\underbrace{\frac{n\rho(s)}{\nm{s}{\cS(\cD)} + 1}}_{r_s}}^{\tilde\blacktriangle},
    \end{aligned}
\end{equation}
where $\bar\cS$ and $\tilde\cS$ are defined as
\begin{align}
    \bar\cS &\coloneqq \left\{ s\in\cS : 0 < \rho(s) < \frac{\log\left(|\cS|(|\cS| + 2)/\delta\right)}{n}\cdot\frac{200}{99}  \right\}, \\
    \tilde\cS &\coloneqq \left\{ s\in\cS : \rho(s) \geq \frac{\log\left(|\cS|(|\cS| + 2)/\delta\right)}{n}\cdot\frac{200}{99}  \right\}.
\end{align}
By the definition of $\bar\cS$, all $s$'s in $\bar\cS$ have small enough $\rho(s)$, and thus $\bar\blacktriangle$ in \eqref{eq:TolerantDecomp} can be trivially bounded from above, i.e.,
\begin{equation}\label{eq:barStateBound}
    \bar\blacktriangle \leq \frac{200}{99}|\cS|\log\frac{|\cS|(|\cS| + 2)}{\delta}.
\end{equation}
For each $s\in\tilde\cS$, we define

$$
\eta_s \coloneqq \sqrt{\frac{2\log\left(|\cS|(|\cS| + 2)/\delta\right)}{n\rho(s)}},
$$

then by the definition of $\tilde\cS$, $1 - \eta_s \geq 0.1$. Therefore, noticing that $\nm{s}{\cS(\cD)}\sim\mathsf{Binomial}\left(n, \rho(s)\right)$, we can apply Corollary~\ref{cor:Bin} to each $r_s$ in \eqref{eq:TolerantDecomp} to conclude that for every $s \in\tilde\cS$, with probability at least $1 -  \nicefrac{\delta}{ (|\cS|(|\cS| + 2)) }$,

\begin{equation}
    \frac{r_s}{n\rho(s)} \leq \frac{1}{(1 - \eta_s)n\rho(s) + 1} \leq \frac{1}{0.1n\rho(s) + 1} \leq \frac{10}{n\rho(s)},
\end{equation}

which followed by a union bound argument within $\tilde\cS$ yields that with probability at least $1 - \nicefrac{\delta}{(|\cS| + 2)}$,

\begin{equation}\label{eq:tildeStateBound}
    \tilde\blacktriangle \text{ in \eqref{eq:TolerantDecomp}} \leq 10|\cS|.
\end{equation}
We then denote by $\tilde E$ the event conditioned on which \eqref{eq:tildeStateBound} holds and denote by $\dot{E}$ the event conditioned on which \eqref{eq:empselDetailedUpperBound} holds. A union bound argument over ${\tilde E}^c$ and ${\dot{ E}}^c$ shows that with probability at least $1 - \delta$, the LHS of \eqref{eq:empselRefined:sentence} is bounded from above by

$$
\frac{4|\cA| + 3\sqrt{|\cA|\log((|\cS| + 2)/\delta)}}{9{n}} \cdot\left( 12|\cS| + \frac{200}{99}|\cS|\log\frac{|\cS|(|\cS| + 2)}{\delta} \right) + \frac{4|\cS|}{9n} + \frac{3\sqrt{|\cS|}}{n}\log\frac{|\cS| + 2}{\delta}.
$$

For $|\cS| = 1$, invoking Lemma~\ref{lem:missingMassCor} for \eqref{eq:tvUpperBoundedByMissingMass} to draw the conclusion.

\end{proof}

\subsubsection{Proof of the Upper Bound in Expectation}

\begin{proof}
    Substituting \eqref{eq:tvUpperBoundedByMissingMass} into \eqref{eq:empselRefined_decomp_first} yields an upper bound for $\EE\condtv{\empsel}{\pi^\star}{\rho}$ as
    \begin{equation}\label{eq:empselExpBound_first}
        \frac{1}{n}\sum_{s\in\cS}\EE v_s + \sum_{s\in\cS}\EE\left[\underbrace{|\rho(s) - \hat\rho(s)|\EE[\mm{\pi^\star(\cdot|s)}{\cA(\cD, s)}|\nm{s}{\cS(\cD)}]}_{\tilde{v}_s}\right].
    \end{equation}
    Each $\EE v_s$ in \eqref{eq:empselExpBound_first} can be bounded from above via \eqref{eq:tvUpperBoundedByMissingMass} by
    \begin{equation}\label{eq:E_v_s_upperBound}
        \EE\left[\underbrace{\nm{s}{\cS(\cD)}\EE[\mm{\pi^\star(\cdot|s)}{\cA(\cD, s)}|\nm{s}{\cS(\cD)}]}_{\bar{v}_s}\right].
    \end{equation}
    Conditioned on $B_{s,0}^c$, invoking Lemma~\ref{lem:visElem} to conclude
    \begin{equation}\label{eq:bar_v_s_ineq}
        \bar{v}_s \leq \frac{4|\cA|}{9}.
    \end{equation}
    The above inequality trivially holds if otherwise conditioned on $B_{s,0}$. Similarly, we always have
    \begin{equation}\label{eq:tilde_v_s_ineq}
        \tilde{v}_s \leq \rho(s)\frac{|\cA|}{\nm{s}{\cS(\cD)} + 1} + \frac{4|\cA|}{9n}.
    \end{equation}
    Substituting \eqref{eq:bar_v_s_ineq} back to \eqref{eq:E_v_s_upperBound} gives
    \begin{equation}\label{eq:empselExpUpper_part_1}
        \frac{1}{n}\sum_{s\in\cS}\EE v_s \lesssim \frac{|\cS||\cA|}{n}.
    \end{equation}
    By Lemma~\ref{lem:BinInvUpperBound}, substituing \eqref{eq:tilde_v_s_ineq} and \eqref{eq:empselExpUpper_part_1} back to \eqref{eq:empselExpBound_first} yields
    \begin{equation}
        \EE\condtv{\empsel}{\pi^\star}{\rho} \lesssim \frac{|\cS||\cA|}{n} + |\cA|\sum_{s\in\cS}\EE\frac{\rho(s)}{\nm{s}{\cS(\cD)} + 1} \lesssim \frac{|\cS||\cA|}{n}.
    \end{equation}
\end{proof}

\begin{remark}\label{rmk:2ndCase:norm_vs_unnorm}
 The empirical variant of vanilla $\mathsf{SEL}$ in our second case actually match the $\log$ probability, which is by definition normalized. Analyzing an unnormalized version, which is more relevant to the matching the logits in practice \citep{ba2014deep,kim2021comparing}, in the second setting may call for new techniques. Also, some preliminary results on the empirical side manifest the difference between minimizing forward $\mathsf{KL}$ and reverse $\mathsf{KL}$ in scenarios related to our last setting \citep{jiang2019smart, gu2023knowledge, agarwal2023gkd}, whose analysis are left are future work.
\end{remark}

\subsection{Proof of Theorem~\ref{thm:full}}\label{subsec:fullklProofs}

\begin{proof}
Since $\mathsf{TV}$ is bounded from above by $1$ and $\tv{\fullkl(\cdot| s)}{\pi^\star(\cdot| s)} = 0, \forall s \in\cS(\cD)$,
    \begin{equation}\label{eq:reduction2MissingMass}
        \text{LHS} = \sum_{s\in\cS\backslash\cS(\cD)}\rho(s)\tv{\fullkl(\cdot| s)}{\pi^\star(\cdot| s)} \leq \sum_{s\in\cS}\rho(s)\ind\{s\notin\cS(\cD)\} \eqqcolon M.
    \end{equation}

    Noticing that $M$ realizes Definition~\ref{def:missingMass} to $\mm{\rho}{\cS(\cD)}$, we invoke Lemma~\ref{lem:missingMassCor} to get
    \begin{equation}\label{eq:fullProxy}
            M \leq \EE M + \frac{3\sqrt{|\cS|}\log(1/\delta)}{n} \leq \frac{4|\cS|}{9n} + \frac{3\sqrt{|\cS|}\log(1/\delta)}{n}.
    \end{equation}
    Substituting \eqref{eq:fullProxy} back to \eqref{eq:reduction2MissingMass} finishes the proof.
\end{proof}

\section{Auxiliary Lemmas}

In contrast with other non-asymptotic tools below, we must assume the mass $p$ does not vary with $n$ in the asymptotic guarantee Lemma~\ref{lem:gc4pmf}.

\begin{lemma}\label{lem:gc4pmf}
    Let $p$ be a probability mass function over an alphabet $\cS$, whose empirical estimation from $X_1, \dots, X_n \stackrel{i.i.d.}{\sim} p$ is $p_n(\cdot) \coloneqq \sum_{i=1}^n \nicefrac{\ind\{X_i = \cdot\} }{ n }$, then
    $$
    p_n \stackrel{a.s.}{\longrightarrow} p,
    $$
    where the almost surely convergence is defined under the $\ell_\infty$ metric in $\RR^{|\cS|}$.
\end{lemma}
\begin{proof}
    Without loss of generality, we assume $\cS = [|\cS|]$; thereby inducing $p(x) = F(x) - F(x-1)$ for $x\in[|\cS|]$, where $F(x) = \PP(X \leq x)$ is the distribution function of $X\sim p$. Similarly, $p_n(x) = F_n(x) - F_n(x-1)$ for 
    $$
    F_n(\cdot) = \sum_{i=1}^n\frac{\PP(X_i\leq\cdot)}{n}.
    $$
    Therefore,
    \begin{align*}
        \max_{x\in[|\cS|]}\abs{ p_n(x) - p(x) } &\leq \sup_{x\in\RR}\abs{ F_n(x) - F(x) - \left(F_n(x-1) - F(x-1)\right) }\\
        &\leq 2\sup_{x\in\RR}\abs{F_n(x) - F(x)} \eqqcolon 2\norm{F_n - F}_\infty.
    \end{align*}
    The proof is thus completed by invoking the Glivenko-Cantelli Theorem \citep[Theorem~19.1]{van2000asymptotic}.
\end{proof}

\subsection{Bounding $\mathsf{TV}$ from Above}

\begin{lemma}[Bretagnolle–Huber inequality {\citep{bretagnolle1979estimation}}]\label{lem:BretagnolleHuberIneq}
    If $P$ and $Q$ are two probability measures on the same measurable space, then
    $$
        \tv{P}{Q} \leq 1 - \frac{1}{2}\exp\left( -\kl{P}{Q} \right).
    $$
\end{lemma}

\begin{lemma}\label{lem:TVempiricalExpectationBound}
    If $a_1, \dots, a_n \stackrel{i.i.d.}{\sim} \pi\in\Delta(\cA)$, whose MLE is $\hat\pi = \hat\pi(a_1, \dots, a_n)$; and $|\cA| < \infty$, then
    $$
    \EE \tv{\hat \pi}{\pi} \leq \frac{1}{2}\sqrt{\frac{|\cA|}{n}}.
    $$
\end{lemma}
\begin{proof}
    We reproduce the proof of this standard result here for completeness.
    \begin{align*}
        \text{LHS} &= \frac{1}{2}\sum_{a\in\cA}\EE\abs{ \hat\pi(a) - \pi(a) } \leq \frac{1}{2}\sum_{a\in\cA}\sqrt{\EE({ \hat\pi(a) - \pi(a) })^2} \\
        &= \frac{1}{2}\sum_{a\in\cA}\sqrt{\frac{1}{n^2}\Var\left(n\hat\pi(a)\right)  } = \frac{1}{2\sqrt{n}}\sum_{a\in\cA}\sqrt{\pi(a)(1 - \pi(a))}\\
        &\leq \frac{1}{2\sqrt{n}}\sum_{a\in\cA}\sqrt{\pi(a)} \leq \frac{1}{2}\sqrt{\frac{|\cA|}{n}},
    \end{align*}
    where the first inequality is by Jensen's inequality, the third equality holds due to $n\hat\pi(a)\sim\mathsf{Binomail}(n, \pi(a))$, and the last inequality is by the Cauchy-Schwarz inequaity.
\end{proof}

\begin{lemma}\label{lem:TVempiricalUpperBound}
    Under the same setting as Lemma~\ref{lem:TVempiricalExpectationBound}, for any $\delta\in(0, 1)$, with probability at least $1-\delta$,
    $$
        \tv{\hat \pi}{\pi} \leq \sqrt{\frac{|\cA|\log 2 + \log(1/\delta)}{2n}}.
    $$
\end{lemma}
\begin{proof}
    This is a straightforward corollary of the Bretagnolle-Huber-Carol inequality {\citep[Proposition~A.6.6]{van1996empirical}} based on the relationship between $\mathsf{TV}$ and $\ell_1$.
\end{proof}

\subsection{Missing Mass Analysis}

Observations like Lemma~\ref{lem:visElem} are key and common in the analysis of learning from finite and static datasets \citep{rajaraman2020toward, rashidinejad2021bridging}.

\begin{lemma}\label{lem:visElem}
    For all $x\in[0, 1], n > 0$, $x(1 - x)^n \leq (\nicefrac{4}{9})n$.
\end{lemma}
\begin{proof}
    By taking the derivative w.r.t. $x$,
    $$
    \max_{x\in[0,1]}\text{LHS} = \left(\frac{n}{n+1}\right)^{n+1}\frac{1}{n} \leq \frac{1}{n}\lim_{n\to\infty}(1 - \frac{1}{n+1})^{n+1} = \frac{1}{en} \leq \frac{4}{9n}.
    $$
\end{proof}

\begin{lemma}[{\citealt[Theorem~A.2]{rajaraman2020toward}}]\label{lem:missingMass}
    Given a distribution $\nu$ on an alphabet $\cS$ and $n$ \emph{i.i.d.} samples $X^n\stackrel{i.i.d.}{\sim} \nu$, then for any $\delta \in (0, \nicefrac{1}{10}]$, with probability at least $1 - \delta$,
    $$
    \mm{\nu}{X^n} - \EE[\mm{\nu}{X^n}] \leq \frac{3\sqrt{|\cS|}\log(1/\delta)}{n}.
    $$
\end{lemma}

\begin{lemma}\label{lem:missingMassCor}
    Under the same setting as Lemma~\ref{lem:missingMass}, for any $\delta \in(0, \nicefrac{1}{10}]$, with probability at least $1 - \delta$,
    $$
    \mm{\nu}{X^n} \leq \frac{4|\cS|}{9n} + \frac{3\sqrt{|\cS|}\log(1/\delta)}{n}.
    $$
\end{lemma}
\begin{proof}
    \begin{equation}\label{eq:expectedMissingMass}
      \begin{aligned}[b]
        \EE \mm{\nu}{X^n} &= \sum_{x\in\cS}\nu(x)\EE \ind\{ x \notin X^n \} = \sum_{x\in\cS}\nu(x)\PP(x \notin X^n)\\
        &= \sum_{x\in\cS}\nu(x)\left(1 - \nu(x)\right)^n \leq \frac{4|\cS|}{9n},
    \end{aligned}  
    \end{equation}
    where the inequality holds due to Lemma~\ref{lem:visElem}; we conclude that $\forall \delta\in(0, \nicefrac{1}{10}]$, with probability at least $1 - \delta$,
    $$
    \mm{\nu}{X^n} \leq \frac{4|\cS|}{9n} + \frac{3\sqrt{|\cS|}\log(1/\delta)}{n}
    $$
    by substituting \eqref{eq:expectedMissingMass} into Lemma~\ref{lem:missingMass}.
\end{proof}

\subsection{Upper Bounds for $\mathsf{Binomial}(n, p)$}

The following two bounds for $X\sim\mathsf{Binomial}(n, p)$ both follow from $\EE[z^X] = (1-p + pz)^n,\forall z\in\RR$.

\begin{lemma}\label{lem:BinInvUpperBound}
    Let $X\sim\mathsf{Binomial}(n, p)$. If $p\in(0, 1]$,
    $$
    \EE\frac{1}{X + 1} \leq \frac{1}{p(n+1)}.
    $$
\end{lemma}
\begin{proof}
    This folklore \citep{canonne2020short} derives from an observation that by Fubini's Theorem,
    $$
        \EE\frac{1}{X + 1} = \int^1_0\EE[z^X]dz,
    $$
    whose RHS is 
    $$
    \int^1_0(1-p + pz)^ndz = \frac{(1-p+pz)^{n+1}}{p(n+1)}\bigg|^1_0 = \frac{1 - (1-p)^{n+1}}{p(n+1)} \leq \frac{1}{p(n+1)}.
    $$
\end{proof}

\begin{lemma}\label{lem:ChernoffBin}
    Let $X\sim\mathsf{Binomial}(n, p)$. For any $\eta\in(0, 1)$,
    $$
    \PP(X \leq (1 - \eta)np) \leq \exp\left( -\frac{\eta^2np}{2} \right).
    $$
\end{lemma}
\begin{proof}
    A combination of \citet[Exercise~4.7]{mitzenmacher2017probability} and the proof of \citet[Theorem~4.5]{mitzenmacher2017probability} yields the upper bound, which we provide here for completeness. For any $t < 0$, by Markov's inequality,
    \begin{align*}
        \PP\left(X \leq (1 - \eta)np\right) &=  \PP\left(e^{tX} \leq e^{t(1 - \eta)np}\right) \leq \frac{\EE[e^{tX}]}{e^{t(1 - \eta)np}} = \frac{\left(1 + p(e^t-1)\right)^n}{e^{t(1 - \eta)np}} \\
        &\leq \frac{\exp\left(np(e^t-1)\right)}{e^{t(1 - \eta)np}} = \left(\frac{\exp\left(e^t-1\right)}{e^{t(1 - \eta)}}\right)^{np} = \left(\frac{e^{-\eta}}{(1-\eta)^{1-\eta}}\right)^{np},
    \end{align*}
    where the last inequality follows from $1 + x \leq e^x, \forall x\in\RR$ and in the last equality we set $t = \log(1-\eta)$. It remains to show
    \begin{equation}\label{eq:ChernoffBinTmp}
        {-\eta} - ({1-\eta})\log(1-\eta) \leq  -\frac{\eta^2}{2} , \forall \eta\in(0, 1).
    \end{equation}
    We thereby define $f(\eta)\coloneqq{-\eta} - ({1-\eta})\log(1-\eta) + 0.5{\eta^2}$. A direct calculation gives
    \begin{align*}
        f^\prime(\eta) &= \log(1 - \eta) + \eta, f^\prime(0) = 0; \\
        f^{\prime\prime}(\eta) &= -\frac{1}{1 - \eta} + 1 < 0, \forall \eta\in(0, 1);
    \end{align*}
    So $f$ is nonincreasing in $[0, 1)$ and thus \eqref{eq:ChernoffBinTmp} holds.
\end{proof}

Lemma~\ref{lem:ChernoffBin} helps us obtain a high-probability counterpart of Lemma~\ref{lem:BinInvUpperBound}.

\begin{corollary}\label{cor:Bin}
    Let $X\sim\mathsf{Binomial}(n, p)$ and $p > 0$. For any $\delta\in(0, 1)$, if
    $$
    \eta = \sqrt{\frac{2\log(1/\delta)}{np}} < 1,
    $$
    then with probability at least $1 - \delta$,
    $$
    \frac{1}{X + 1} \leq \frac{1}{(1 - \eta)np + 1}.
    $$
\end{corollary}
\begin{proof}
    By Lemma~\ref{lem:ChernoffBin},
    \begin{equation*}
        \PP\left( \frac{1}{X + 1} > \frac{1}{(1 - \eta)np + 1} \right) \leq \PP\left( X \leq (1 - \eta)np\right) \leq \exp\left( -\frac{\eta^2np}{2} \right) = \delta.
    \end{equation*}
\end{proof}

\section{Hard-to-Learn Instances for Experiments}
\label{appendix:ExperimentsSepc:subsec_instances}

\paragraph{Instance 0}\label{instance:0} For every $s\in\cS$, $\pi^\star(\cdot|s) \coloneqq 0.5\mathsf{Uniform}(\cA) + 0.5 \mathsf{Dirac}(\cA, s \mod A + 1)$ because any reference policy far away from both $\mathsf{Uniform}(\cA)$ and any one in $\mathsf{Dirac}(\cA)$ is sufficient to reveal the disadvantage of $\empce$ according to Theorem~\ref{thm:empce_bias}. $\rho \coloneqq \mathsf{Uniform}(\cS)$ is enough to ensure about $\nicefrac{n}{S}$ visitations of each input.

Interestingly, we conjecture there does not exist a worst-of-three-worlds instance that can simultaneously expose the fundamental limits of $\nll$, $\empsel$, and $\fullkl$, in that the constructive proofs (in Appendix~\ref{sec:appendix:lbs}) of Theorem~\ref{thm:lb1}, Theorem~\ref{thm:lb2}, and Theorem~\ref{thm:lb3} since the progressively richer information are substantially different from each other. Since our learners in this section is uniformly initialized over unseen labels, any single instance covered by the Bayes prior in the lower bound arguments of a setting\footnote{See, e.g., Appendix~\ref{sec:appendix:lb2proof} for a concrete Bayes prior in use.} is sufficient to numerically illustrate the corresponding difficulty of estimation (learning).

\paragraph{Instance 1}\label{instance:1} To verify the minimax optimality of $\nll$ with only samples avaiable, we adapt the proof of Theorem~\ref{thm:lb1} (Appendix~\ref{sec:appendix:lb1proof}), which is based on Assouad's hypercube reduction \citep{yu1997assouad}. In numerical simlations, any vertex of the hypercube is applicable since we have already enforced an uniform initialization of any $\hat\pi$ in unseen inputs. We choose the teacher policy
\begin{equation*}
        \pi^\star(2j - 1|s) = \frac{1 + 0.25\sqrt{SA/n}}{A}, \pi^\star(2j|s) = \frac{1 - 0.25\sqrt{SA/n}}{A}, \forall (s,j)\in\overline{[S]}\times\left[\frac{A}{2}\right];
\end{equation*}
and $\rho = \mathsf{Uniform}(\cS)$ for simplicity. The two key insights behind the design of Instance~\hyperref[instance:1]{1} are (1) $\rho$ must be nonvanishing in $\Omega(|\cS|)$ inputs to manifest the hardness of $|\cS| > 0$, (2) $\tv{\pi^\star(\cdot|s)}{\mathsf{Uniform}(\cA)} = \Theta(n^{-0.5})$ is crucial for Instance~\hyperref[instance:1]{1} to be hard enough for any minimax optimal learner. (If $|\cA|$ is odd, simply let the last label $A$ to have zero mass and replace $A$ with $A - 1$ here.)

\paragraph{Instance 2}\label{instance:2} To verify the minimax optimality of $\empsel$ with sampled odds avaiable, we adapt the proof of Theorem~\ref{thm:lb2} (Appendix~\ref{sec:appendix:lb2proof}), which is based on a carefully designd Bayes prior. Similarly, we can use any single instance covered by the support of the Bayes prior. We choose the teacher policy
$$
\pi^\star(2j - 1|s) = \frac{S}{n+1}, \pi^\star(2j|s) = 0, \pi^\star(A|s) = 1 - \frac{S}{2}\cdot\frac{A-1}{n+1}, \forall (s,j)\in\overline{[S]}\times\left[\frac{A-1}{2}\right];
$$
and $\rho = \mathsf{Uniform}(\cS)$ for simplicity. (If $|\cA|$ is even, simply let the last label $A$ to have zero mass and replace $A$ with $A - 1$ here.)

\paragraph{Instance 3}\label{instance:3} To verify the minimax optimality of $\fullkl$ with complete logits avaiable, we adapt the proof of Theorem~\ref{thm:lb3} (Appendix~\ref{sec:appendix:lb3proof}), which includes a specialized $\rho$ to slow down the convergence of $\fullkl$. Specifically, $\rho = (n+1)^{-1}$ for all inputs except the last one. Theoretically, the assignment of $\pi^\star$ will not affect the convergence of $\fullkl$, so we use a $\pi^\star$ same as that in Instance~\hyperref[instance:3]{3} only to ensure that $\empsel$ is not able to converge too fast.

\section{Discussions: Dependent Samples in Rewardless MDPs}
\label{sec:extToTransKer}
Besides the popular approach of fine-tuning LLMs \citep{ouyang2022training,touvron2023llama2} that interprets instructions as inputs and the entire response as a label, there is a more granular perspective where each token is considered a label $a_i$ (See, e.g., the \verb|logprobs| option in the OpenAI completion API\footnote{\url{https://platform.openai.com/docs/api-reference/completions/create\#logprobs}}.) and $s_{i+1}$ is simply the concatenation of $s_i$ and $a_i$, i.e., $s_{i+1} \sim \PP(\cdot|s_i,a_i)$, where $\PP$ is the deterministic transition kernel induced by concatenation. Our bounds for i.i.d. samples \textbf{already} subsume this plausible more involved case through \emph{lack of reward}. The following reductions hold for any $\PP \in\Delta(\cA|\cS\times\cA)$ including the aforementioned concatenation kernel.

The proof of any lower bound remains valid so long as $\PP(\cdot|s,a) \coloneqq \rho(\cdot), \forall (s,a)\in\cS\times\cA$ in the constructed hard-to-learn instance, making the samples i.i.d. Our upper bounds allow $\rho$ to explicitly depend on $\pi^\star$ and even $\PP$, so the samples can be viewed as i.i.d. samples from $d_{\pi^\star}^\PP\times\pi^\star$ given the \emph{input occupancy measure} $d_{\pi^\star}^\PP \in \Delta(\cS)$ is well-defined. Hence, replacing $\rho$ with $d_{\pi^\star}^\PP$ validates all arguments for upper bounds. Intuitively, $d_{\pi^\star}^\PP(s)$ is the probability of visiting $s$ in a trajectory induced by the transition kernel $\PP$ and reference policy $\pi^\star$. It is well-developed in either episodic MDPs \citep{yin2021near} or discounted MDPs \citep{rashidinejad2021bridging}. These seamless reductions crucially hinge on the absence of value functions and any notion of reward signal in our theoretical framework.

\begin{remark}
    Our analysis covers but is not specialized to the case where $\rho$ depends on $\pi^\star$ or vice versa. Therefore, the result remains unchanged regardless of the relation between $\rho$ and the original training set for training the teacher $\pi^\star$. (For example, $\rho$ may be the distribution of instructions selected by maintainers on the student side \citep{peng2023instruction}.) It will be intriguing if some further analysis can show any impact of teacher training or data quality on the students' statistical rate.
\end{remark}

\section{Dicussions on Function Approximation}\label{sec:funcApprox}

\subsection{Log-linear and General Softmax Conditional Densities}

We discuss potential ways and obstacles of generalizing the results above to large or even uncountable (continuous) state space. First, we extend the concept of conditional probability space $\Delta(\cdot|\cdot)$ to general spaces rigorously. In the following discussions, we assume the notation $\cY$ and $\cA$ refer to finite sets for simplicity. $\Delta(\cdot)$ in this section receives any standard Borel space as input and returns the set of probability measures on it. 

\begin{definition}[{\citealt[Definition~2.8]{polyanskiy2022information}}]\label{def:generalMarkovKernel}
    Given two standard Borel spaces $\cX, \cY$, a conditional probability (the teacher/student we considered in this paper) $\pi:\cX\to\cY$ is a bivariate function $\pi(\cdot|\cdot)$, whose first argument is a measurable subset of $\cY$ and the second is an element of $\cX$, such that:
    \begin{itemize}
        \item $\forall x\in\cX$, $\pi(\cdot|x)$ is a probability measure on $\cY$, and
        \item $\forall \text{ measurable } A\subset\cY$, $x\to\pi(A|x)$ is a measurable function on $\cX$.
    \end{itemize}
\end{definition}

The following preliminary result (whose proof is deferred to Section~\ref{subsec:Proof4tvLeqNorm_LogLinear}) may shed light on prospective approaches to the analysis of function approximation.

\begin{proposition}\label{lem:tvLeqNorm_LogLinear}
    For standard Borel spaces $\cX$ and $\cY$, if both $\acute \pi, \grave\pi\in\Delta(\cY|\cX)$ are log-linear, i.e., $\acute \pi = \Pi(\bphi, \acute\btheta), \grave\pi = \Pi(\bphi, \grave\btheta)$, where
    \begin{equation*}
        \Delta(\cY|\cX) \ni \Pi(\bphi, \btheta) \propto \exp\left(\inner{\bphi}{\btheta}\right), \bphi: \cX\times\cY \to \RR^d;
    \end{equation*}
    and $\sup_{(x, y) \in\cX\times\cY}\norm{\bphi(x, y)}_2 \leq M$; then for any $\nu \in \Delta(\cX)$,
    \begin{equation}\label{eq:tvLeqNorm_LogLinear}
        \condtv{\acute\pi}{\grave\pi}{\nu} \leq M\norm{\acute\btheta - \grave\btheta}_2.
    \end{equation}
\end{proposition}

\subsection{Take-Home Messages and Conjectures}

Obviously, any analysis leveraging Proposition~\ref{lem:tvLeqNorm_LogLinear} can potentially generalize our results, since log-linear $\pi^\star$ subsumes tabular $\pi^\star$. Technically, \eqref{eq:tvLeqNorm_LogLinear} mainly hinges on the (uniform) $M$-Lipschitz continuity of $\inner{\bphi}{\btheta}$ with respect to $\btheta$ for any $\btheta\in\RR^d$, therefore, it is also conceptually straightforward to extend Proposition~\ref{lem:tvLeqNorm_LogLinear} to general Sofxmax $\pi^\star$, which we omit here for brevity.

Based on \eqref{eq:tvLeqNorm_LogLinear}, we conjecture that a fine-grained analysis of the $\ell_2$-norm of $\hat\btheta - \btheta^\star$ may be the key to \textbf{bound} $ \condtv{\hat\pi}{\pi^\star}{\rho}$ \textbf{from above} for any $\hat\pi, \pi^\star\in\Delta(\cA|\cS)$ and $\rho\in\Delta(\cS)$. Since the tabular setting is a special case of the log-linear setting, we also conjecture that $\norm{\hat\btheta - \btheta^\star}_2 \gtrsim \sqrt{\nicefrac{d}{n}}$ via \dataone{} and $\norm{\hat\btheta - \btheta^\star}_2 \gtrsim \nicefrac{d}{n}$ via \datatwo{}. 

\subsection{Proof of Proposition~\ref{lem:tvLeqNorm_LogLinear}}\label{subsec:Proof4tvLeqNorm_LogLinear}

\begin{proof}[Proof of Proposition~\ref{lem:tvLeqNorm_LogLinear}]
    For any $x\in\cX$,
    \begin{equation}\label{eq:klLeqtvNorm}
        \begin{aligned}[b]
            \kl{\acute\pi(\cdot|x)}{\grave\pi(\cdot|x)} &= \sum_{y\in\cY}\acute\pi(y|x)\log\frac{\acute\pi(y|x)}{\grave\pi(y|x)}\\
            &= \sum_{i\in\cY}\acute\pi(i|x)\left[ \inner{\bphi(x, i)}{\acute\btheta - \grave\btheta} + \log\frac{ \sum_{k\in\cY}\exp\left(\inner{\bphi(x, k)}{\grave\btheta}\right) }{ \sum_{j\in\cY}\exp\left(\inner{\bphi(x,j)}{\acute\btheta}\right) } \right] \\
            &\leq \sum_{i\in\cY}\acute\pi(i|x)\inner{\bphi(x, i)}{\acute\btheta - \grave\btheta} + \sum_{j\in\cY}\frac{\exp\left(\inner{\bphi(x,j)}{\grave\btheta}\right)}{\sum_{k\in\cY}\exp\left(\inner{\bphi(x, k)}{\grave\btheta}\right)}\inner{\bphi(x, j)}{\grave\btheta - \acute\btheta}\\
            &= \sum_{i\in\cY}\left( \acute\pi(i|x) - \grave\pi(i|x) \right)\inner{\bphi(x, i)}{\acute\btheta - \grave\btheta} \\
            &\leq \sum_{i\in\cY} \abs{\acute\pi(i|x) - \grave\pi(i|x)}\norm{\bphi(x, i)}_2\norm{\acute\btheta - \grave\btheta}_2 \\
            &\leq 2\tv{\acute\pi(\cdot|x)}{\grave\pi(\cdot|x)}\cdot M \cdot \norm{\acute\btheta - \grave\btheta}_2,
        \end{aligned}
    \end{equation}
    where the first inequality holds due to the log-sum inequality, the second inequality is a combination of triangle inequality and Cauchy–Schwarz inequality, and the last inequality is by the boundedness of $\bphi$ together with the well-known $2\mathsf{TV} = \ell_1$ relation.
    We plug \eqref{eq:klLeqtvNorm} into Pinsker's inequality to obtain
    $$
    \left[\tv{\acute\pi(\cdot|x)}{\grave\pi(\cdot|x)}\right]^2 \leq \frac{1}{2}\kl{\acute\pi(\cdot|x)}{\grave\pi(\cdot|x)} \leq M \tv{\acute\pi(\cdot|x)}{\grave\pi(\cdot|x)}\norm{\acute\btheta - \grave\btheta}_2.
    $$
    So $\condtv{\acute\pi}{\grave\pi}{\nu} = \int_{\cX}\tv{\acute\pi(\cdot|x)}{\grave\pi(\cdot|x)}\ud\nu \leq M\norm{\acute\btheta - \grave\btheta}_2$.
\end{proof}

\end{document}